\documentclass[12pt]{colt2020} 

\title{Empirical Policy Evaluation with Supergraphs}
\usepackage{times}
\coltauthor{%
 \Name{Daniel Vial*} \Email{dvial@utexas.edu}\\
 \addr University of Texas at Austin
 \AND
 \Name{Vijay Subramanian} \Email{vgsubram@umich.edu}\\
 \addr University of Michigan
}

\usepackage{mathtools}
\mathtoolsset{showonlyrefs}
\DeclareMathOperator*{\argmax}{arg\,max}
\DeclarePairedDelimiter{\ceil}{\lceil}{\rceil}
\newcommand{\E}{\mathbb{E}}
\renewcommand{\P}{\mathbb{P}}
\newcommand{\R}{\mathbb{R}}
\newcommand{\N}{\mathbb{N}}
\newcommand{\Z}{\mathbb{Z}}

\newcommand{\tr}{\mathsf{T}}
\renewcommand{\epsilon}{\varepsilon}

\begin{document}

\maketitle

\begin{abstract}%
We devise and analyze algorithms for the empirical policy evaluation problem in reinforcement learning. Our algorithms explore backward from high-cost states to find high-value ones, in contrast to forward approaches that work forward from all states. While several papers have demonstrated the utility of backward exploration empirically, we conduct rigorous analyses which show that our algorithms can reduce average-case sample complexity from $O(S \log S)$ to as low as $O(\log S)$.
\end{abstract}


\section{Introduction} \label{secIntro}

Reinforcement learning (RL) is a machine learning paradigm with potential for impact in wide-ranging applications. At a high level, RL studies autonomous agents interacting with uncertain environments -- by taking actions, observing the effects of those actions, and incurring costs -- in hopes of achieving some goal. Mathematically, this is often cast in the following (finite, discrete-time) Markov decision process (MDP) model. Let $\mathcal{S}$ and $\mathcal{A}$ be finite sets of states and actions, respectively; for simplicity, we assume $\mathcal{S} = \{1,\ldots,S\}$ for some $S \in \N$ throughout the paper. The uncertain environment is modeled by a controlled Markov chain with transition matrix $Q$, i.e.\
\begin{equation}
\P ( S_{t+1} = s' | S_t = s , A_t = a ) = Q( s' | s,a )\ \forall\ s,s' \in \mathcal{S}, a \in \mathcal{A} ,
\end{equation}
where $\{ S_t \}_{t=0}^{\infty}$ and $\{ A_t \}_{t=0}^{\infty}$ are the random sequences of states and actions, respectively. State-action pair $(s,a) \in \mathcal{S} \times \mathcal{A}$ incurs instantaneous cost $c(s,a) \in \R_+$. Mappings $\pi : \mathcal{S} \rightarrow \mathcal{A}$ are called (stationary, deterministic, Markov) \textit{policies} and dictate the action taken at each state, i.e.\ $A_t = \pi(S_t)$. If the initial state is $s \in \mathcal{S}$ and the agent follows policy $\pi$, it incurs discounted cost
\begin{equation}\label{eqPolicyEval}
v_{\pi}(s) = \E_{\pi} \left[ (1-\alpha)  \sum_{t=0}^{\infty} \alpha^t c ( S_t, A_t ) \middle| S_0 = s \right] = (1-\alpha) \sum_{t=0}^{\infty} \alpha^t Q_{\pi}^t(s,\cdot) c_{\pi} ,
\end{equation}
where $\E_{\pi}$ means $A_t = \pi(S_t)$ inside the expectation, $\alpha \in (0,1)$ is a \textit{discount factor}, and $Q_{\pi}(s,s') = Q(s'|s,\pi(s))$ and $c_{\pi}(s) = c(s,\pi(s))$ are the transition matrix and cost vector induced by $\pi$. 

To find good policies -- roughly, $\pi$ for which $v_{\pi} = \{ v_{\pi}(s) \}_{s \in \mathcal{S} }$ is small -- one often needs to estimate $v_{\pi}$ for a fixed policy $\pi$. For example, the \textit{empirical policy iteration} algorithm of \cite{haskell2016empirical} iteratively estimates $v_{\pi}$ and greedily updates $\pi$. Moving forward, we focus on the former step, called \textit{empirical policy evaluation} (EPE). The policy $\pi$ will thus be fixed for the remainder of the paper, so we dispense with this subscript in \eqref{eqPolicyEval} and (with slight abuse of notation) define our problem as follows. Let $\alpha \in (0,1)$ be a discount factor, $c \in \R_+^S$ a cost vector, and $Q$ an $S \times S$ row stochastic matrix. We seek an algorithm to estimate the \textit{value function}
\begin{equation}\label{eqValFunc}
v = (1-\alpha) \sum_{t=0}^{\infty} \alpha^t Q^t c = (1-\alpha) ( I - \alpha Q )^{-1} c .
\end{equation}
As is typical in the RL literature, we assume $Q$ is unknown but the agent can sample random states distributed as $Q(s,\cdot)$ via interaction with the environment. Since this interaction can be costly in applications, we aim to estimate \eqref{eqValFunc} with as few samples as possible. In contrast to some works, we also assume $c$ is a known input to the algorithm. Thus, our work is suitable for goal-oriented applications where one knows instantaneous costs \textit{a priori} -- for instance, which states correspond to winning or losing if the MDP models a game -- and aims to estimate long-term discounted costs -- for instance, how good or bad non-terminal configurations of the game are.

To contexualize our contributions, we contrast two approaches to EPE. The first approach is one of forward exploration, where $v$ is estimated by sampling trajectories beginning at each state. We focus on a typical approach employed in e.g.\ \cite{haskell2016empirical}, which we refer to as the \textit{forward approach} for the remainder of the paper, and which proceeds as follows. First, let $\{ W_t \}_{t=0}^{\infty}$ be a Markov chain with transition matrix $Q$, fix $s \in \mathcal{S}$ and $T \in \N$, and rewrite \eqref{eqValFunc} as
\begin{equation}\label{eqValFuncError}
v(s) = (1-\alpha) \sum_{t=0}^{T-1} \alpha^t \E [ c(W_t) | W_0 = s ] + O \left( \| c \|_{\infty} \alpha^T \right) .
\end{equation}
Here the $O(\| c \|_{\infty} \alpha^T)$ bias can be made small if $T$ is chosen large, and the first term can be estimated by simulating length-$T$ trajectories. More specifically, let $\{ W_t^{s,i} \}_{t=0}^{T-1}$ be a trajectory obtained as follows: set $W_0^{s,i} = s$ and, for $t \in \{1,\ldots,T-1\}$, sample $W_t^{s,i}$ from $Q(W_{t-1}^{s,i},\cdot)$. Letting $m \in \N$ and repeating this for $i \in \{1,\ldots,m\}$, we obtain an unbiased estimate of the first term in \eqref{eqValFuncError}:
\begin{equation}
\frac{1}{m} \sum_{i=1}^m (1-\alpha) \sum_{t=0}^{T-1} \alpha^t c ( W_t^{s,i} ) .
\end{equation}
This forward approach is analytically quite tractable; indeed, rigorous guarantees follow easily from standard Chernoff bounds (see Appendix \ref{appStandardApproach}). However, since trajectories must be sampled starting at each state, $\Omega(S)$ samples are fundamentally required, which may be prohibitive in practice.

The second approach we consider is one of backward exploration. This approach relies on the idea that if there are only a few high-cost states with only a few trajectories leading to them, it is more efficient to work backward along just these trajectories (or along a small set containing them) to identify high-value states (those $s$ for which $v(s)$ is large). Put differently, if $Q$ and $c$ are sparse, intuition suggests that backward exploration from high-cost states is more sample-efficient than forward exploration from all states. While intuitively reasonable, there are two issues that  prevent backward exploration from reducing the linear sample complexity of the forward approach. First, the agent must identify high-cost states in order to explore backward from them, without visiting all states. Second, the agent must explore a small set of trajectories likely to lead to high-cost states, without starting at each state and filtering out trajectories that do not reach the high-cost set. Several approaches have been proposed to combat these issues. For instance, \cite{goyal2018recall} uses observed state-action-cost sequences to train a model that generates samples of state-action pairs likely to lead to a given state. This allows the agent to construct simulated trajectories that are guaranteed to lead to high-cost states, addressing the second issue; the observed sequences are also used to identify high-cost states, addressing the first issue. \cite{edwards2018forward} similarly trains a model that predicts which trajectories lead to high-cost states while assuming costs are known \textit{a priori}. In a different vein, \cite{florensa2017reverse} considers physical tasks like a robot navigating a maze which have clear goal states, addressing the first issue. The state-action space is assumed to have a certain continuity -- namely, ``small'' actions (e.g.\ a robot moving a small distance) lead to ``nearby'' states (e.g.\ physically close locations) -- addressing the second issue. 

Our approach is as follows. First, as mentioned above, we assume the cost vector is known \textit{a priori} (like \cite{edwards2018forward} and similar to \cite{florensa2017reverse}). Second, we assume the agent is provided certain side information: $A \in \{0,1\}^{S \times S}$ satisfying the ``absolute continuity'' condition
\begin{equation}\label{eqAbsContCond}
A(s,s') = 0  \Rightarrow  Q (s,s') = 0\ \forall\ s,s' \in \mathcal{S} .
\end{equation}
Note we can view $A$ as the adjacency matrix for a graph whose edges are a superset of those in the graph induced by $Q$; thus, we refer to this side information as the \textit{supergraph}. The utility of the supergraph is that it allows the agent to determine which states may be ``close'' to high-cost states in the induced graph, which may allow for construction of trajectories leading to such states. In this work, we assume the supergraph is provided and do not address the important practical consideration of how to actually obtain it. However, we do note it can likely be obtained from domain knowledge. For instance, in a robot navigation task like \cite{florensa2017reverse}, one-step transitions between physically distant states $s$ and $s'$ may be impossible, which would allow us to conclude $Q(s,s') = 0$ \textit{a priori} and set $A(s,s') = 0$. Unlike \cite{florensa2017reverse}, however, our supergraph assumption does not depend on state-action continuity and thus should hold more generally; for example, if the MDP models a game, the game's rules may prevent transitions from $s$ to $s'$, so that $A(s,s') = 0$. We also emphasize that the reverse of the implication in \eqref{eqAbsContCond} need not hold. Thus, one can always set $A(s,s') = 1\ \forall\ s,s'$ to ensure that \eqref{eqAbsContCond} holds. Of course, there is a trade off; as will be seen, our algorithms are most efficient when $A$ is sparse in a certain sense.

In the remainder of the paper, we devise two backward exploration-based EPE algorithms that exploit the supergraph. Unlike \cite{goyal2018recall,edwards2018forward,florensa2017reverse}, which only present empirical results, our algorithms are amenable to rigorous accuracy and sample complexity guarantees. Thus, our main contribution is to offer theoretical evidence for the empirical success of backward exploration. More precisely, our contributions are as follows. First, we devise an algorithm called \texttt{Backward-EPE} in Section \ref{secBackward} that uses the supergraph to discover high-value states while working backward from high-cost ones. We establish $l_{\infty}$ accuracy and worst-case sample complexity $O(S \log S)$, equivalent to the average-case complexity of the forward approach. More notably, we show the average-case complexity of \texttt{Backward-EPE} is $O( \bar{d} \| c \|_1 / \| c \|_{\infty} \log S)$, where $\bar{d}$ is the average degree in the supergraph. Note this bound precisely captures the intuition that backward exploration depends on how many high-cost states are present ($\| c \|_1 / \| c \|_{\infty}$ term) and how many trajectories lead to them ($\bar{d}$ term). In the extreme case, $\bar{d} \| c \|_1 / \| c \|_{\infty} = O(1)$, in which case \texttt{Backward-EPE} reduces complexity from $S \log S$ to $\log S$. Next, we combine \texttt{Backward-EPE} with the forward approach for our second algorithm \texttt{Bidirectional-EPE} in  Section \ref{secBidirectional}. We establish a (pseudo)-relative error guarantee, which we argue is useful in e.g.\ empirical policy iteration. Analytically, we show \texttt{Bidirectional-EPE} reduces the sample complexity of a plug-in method with the same accuracy guarantee; empirically, we show it is more efficient than using the backward or forward approach alone. Both of our algorithms are inspired by methods that estimate \textit{PageRank}, a node centrality measure used in the network science literature \cite{page1999pagerank}. While seemingly unrelated to EPE, PageRank has mathematical form similar to that of the value function \eqref{eqValFunc}; however, the PageRank estimation literature assumes $Q$ is known, so the extension to EPE is non-trivial. Thus, another contribution of this work to show how PageRank estimators can be adapted to EPE. We believe our algorithms and analysis are only examples of a more general approach; Section \ref{secFuture} discusses other problems where we believe our approach will be useful.

\textbf{Commonly-used notation:} For a matrix $B$ and any $t \in \N$, $B^t(s,s')$, $B^t(s,\cdot)$, and $B^t(\cdot,s')$ denote the $(s,s')$-th entry, $s$-th row, and $s'$-th column of $B^t$, respectively. We write $0_{n \times m}$ and $1_{n \times m}$ for the $n \times m$ matrices of zeroes and ones, resp. Matrix transpose is denoted by $^\tr$. We use $1(\cdot)$ for the indicator function, i.e.\ $1(E) = 1$ if statement $E$ is true and $1(E) = 0$ otherwise. For $s \in \mathcal{S}$, $e_s$ is the $S$-length vector with $1$ in the $s$-th entry and $0$ elsewhere, i.e.\ $e_s(s') = 1(s=s')$. Also for $s \in \mathcal{S}$, $N_{in}(s) = \{ s' \in \mathcal{S} : A(s',s) = 1 \}$ and $d_{in}(s) = |N_{in}(s)|$ are the incoming neighbors and in-degree of $s$ in the supergraph. Average degree is denoted by $\bar{d} = \frac{1}{S} \sum_{s,s'=1}^S A(s,s') = \frac{1}{S} \sum_{s=1}^S d_{in}(s)$. For $\{ a_n \}_{n \in \N} , \{ b_n \}_{n \in \N} \subset [0,\infty)$, we write $a_n = O(b_n)$, $a_n = \Omega(b_n)$, $a_n = \Theta(b_n)$, and $a_n = o(b_n)$, resp., if $\limsup_{n \rightarrow \infty} \frac{a_n}{b_n} < \infty$, $\liminf_{n \rightarrow \infty} \frac{a_n}{b_n}  > 0$, $a_n = O(b_n)$ and $a_n = \Omega(b_n)$, and $\lim_{n \rightarrow \infty} \frac{a_n}{b_n} = 0$, resp. All random variables are defined on a common probability space $(\Omega,\mathcal{F},\P)$, with $\E [ \cdot ]  = \int_{\Omega} \cdot\ d \P$ denoting expectation and $a.s.$ meaning $\P$-almost surely.

\section{Backward empirical policy evaluation} \label{secBackward}


Our first algorithm is called \texttt{Backward-EPE} and is based on the \texttt{Approx-Contributions} PageRank estimator from \cite{andersen2008local}. The latter algorithm restricts to the case $c = e_{s^*}$ for some $s^* \in \mathcal{S}$ and assumes $Q$ is known; our algorithm is a fairly natural generalization to the case $c \in \R_+^S$ and unknown $Q$. For brevity, we restrict attention to \texttt{Backward-EPE} in this section. For transparency, Appendix \ref{appComparison} discusses \texttt{Approx-Contributions} and clarifies which aspects of our analysis are borrowed from \cite{andersen2008local} and other existing work.

\texttt{Backward-EPE} is defined in Algorithm \ref{algBackwardEPE}. The algorithm takes as input cost vector $c$, discount factor $\alpha$, and desired accuracy $\epsilon$, and initializes four variables: a value function estimate $\hat{v}_0 = 0_{S \times 1}$, a residual error vector $r_0 = c$, a set $U_0 = \emptyset$ we call the \textit{encountered set}, and a transition matrix estimate $\hat{Q}_0 = 0_{S \times S}$. Conceptually, the algorithm then works backward from high-cost states, iteratively pushing mass from residual vector to estimate vector so as to improve the estimate of $v$. More precisely, the first iteration proceeds as follows. First, a high-cost state $s_1$ is chosen ($s_1 \in \mathcal{S}$ such that $r_0(s_1) = c(s_1)$ is maximal) and its incoming supergraph neighbors $N_{in}(s_1)$ are added to the encountered set (first line in \texttt{while} loop). For $s \in U_1 = N_{in}(s_1)$ -- i.e.\ $s$ for which $Q(s,s_1)$ may be nonzero by \eqref{eqAbsContCond} -- an estimate $\hat{Q}_1(s,\cdot)$ of $Q(s,\cdot)$ is computed using $n$ samples (first \texttt{for} loop). The estimate $\hat{v}_1(s_1)$ is then incremented with the $(1-\alpha) r_0(s) = (1-\alpha) c(s)$ component of $v(s_1)$, and $\hat{Q}_1(s,s_1)$ is used to estimate the $Q(s,s_1) r_0(s_1) = Q(s,s_1) c(s_1)$ component of $v(s)$ and to increment the corresponding residual $r_1(s)$ (second \texttt{for} loop). 

In subsequent iterations $k$, the iterative update proceeds analogously, choosing $s_k$ to maximize $r_{k-1}(s_k)$, adding $N_{in}(s_k)$ to the encountered set, incrementing $\hat{v}_k(s_k)$ by $(1-\alpha) r_{k-1}(s)$, and using an estimate of $Q(s,s_k) r_{k-1}(s_k)$ to increment $r_k(s)$. The only distinction is that at iteration $k$, $Q(s,\cdot)$ is estimated only for states $s \in U_k \setminus U_{k-1}$. Put differently, the first time we encounter state $s$ -- i.e.\ the first $k$ for which $s \in N_{in}(s_k)$ -- we estimate $Q(s,\cdot)$; we then retain that estimate for the remainder of the algorithm. Thus, the encountered set $U_k$ tracks the rows of $Q$ we have estimated up to and including iteration $k$. Alternatively, one could estimate $Q(s,s_k)$ with independent samples at each iteration $k$ for which $s \in N_{in}(s_k)$; we discuss the merits of this approach in Section \ref{secFuture}.

\begin{algorithm}
\caption{ \texttt{Backward-EPE} } \label{algBackwardEPE}
\KwIn{Sampler for transition matrix $Q$; cost vector $c$; discount factor $\alpha$; supergraph in-neighbors $\{N_{in}(s)\}_{s=1}^S$; termination parameter $\epsilon$; per-state sample count $n$}

$k = 0$, $\hat{v}_k = 0_{S \times 1}$, $r_k = c$, $U_k = \emptyset$, $\hat{Q}_k = 0_{S \times S}$

\While{$\|r_k\|_{\infty} > \epsilon$}{

	$k \leftarrow k+1$, $s_k \sim \argmax_{s \in \mathcal{S}} r_{k-1}(s)$ uniformly, $U_k = U_{k-1} \cup N_{in}(s_k)$ 

	\For{$s \in \mathcal{S}$}{ 
		\lIf{$s \in N_{in}(s_k) \setminus U_{k-1}$}{$\{ X_{s,i} \}_{i=1}^n \sim Q(s,\cdot)$, $\hat{Q}_k(s,\cdot) = \frac{1}{n} \sum_{i=1}^n 1 ( X_{s,i} = \cdot)$}
		\lElse{$\hat{Q}_k(s,\cdot) = \hat{Q}_{k-1}(s,\cdot)$} 
	}

	\For{$s \in \mathcal{S}$}{
		\lIf{$s = s_k$}{$\hat{v}_k(s) = \hat{v}_{k-1}(s) + (1-\alpha) r_{k-1}(s)$, $r_k(s) = \alpha \hat{Q}_k(s,s_k) r_{k-1}(s_k)$} 
		\lElse{$\hat{v}_k(s) = \hat{v}_{k-1}(s)$, $r_k(s) = r_{k-1}(s) + \alpha \hat{Q}_k(s,s_k) r_{k-1}(s_k)$} 
	}
}
\KwOut{Estimate $\hat{v}_{k_*} = \hat{v}_k$ of $v = (1-\alpha) \sum_{t=0}^{\infty} \alpha^t Q^t c$}
\end{algorithm}

The manner in which we estimate $Q$ and update the estimate and residual vectors may appear mysterious, but it allows us to prove the following  analogue of a key result from \cite{andersen2008local}. To explain this result, first let $k_* = \inf \{ k \in \Z_+ : \|r_k\|_{\infty} \leq \epsilon \}$ denote the iteration at which \texttt{Backward-EPE} terminates, and let $\hat{\mu}_s$ denote the $s$-th row of $(1-\alpha) (I-\alpha \hat{Q}_{k_*})^{-1}$, so that $\hat{\mu}_s c$ is the value function for $s \in \mathcal{S}$ defined on the final estimate $\hat{Q}_{k_*}$ of $Q$. Then the result (roughly) says that the fixed point equation $\hat{v}_k(s) + \hat{\mu}_s r_k = \hat{\mu}_s c$ is preserved across iterations $k \in \{0,\ldots,k_*\}$. Conceptually, this means that if we run the algorithm until it terminates to obtain $\hat{Q}_{k_*}$, then look back at the sequence $\{ \hat{v}_k , r_k \}_{k=0}^{k_*}$ generated by the algorithm, the fixed point equation will have held at each $k$. This non-causality is somewhat unintuitive, yet is crucial to the ensuring analysis.

More precisely, Lemma \ref{lemInvariantUnknownQ} says that such a fixed point equation holds for certain row stochastic matrices which differ from $\hat{Q}_{k_*}$ only in unestimated rows of $Q$, i.e.\ rows indexed by $\mathcal{S} \setminus U_{k_*}$. The set of such matrices for which the result holds is discussed in Appendix \ref{proofLemInvariantUnknownQ}; for brevity, here we state the result only for the two elements of this set we require in later analyses: $\overline{Q}$, which fills unestimated rows with offline estimates, and $\underline{Q}$, which fills unestimated rows with actual rows of $Q$. 

\begin{lemma} \label{lemInvariantUnknownQ}
Let $Y_{s,i} \sim Q(s,\cdot)\ \forall\ s \in \mathcal{S}, i \in [n]$, independent across $s$ and $i$, and independent of the random variables in Algorithm \ref{algBackwardEPE}. From $\{ Y_{s,i} \}_{s \in \mathcal{S}, i \in [n]}$, define an offline estimate $\tilde{Q}$ of $Q$ row-wise by $\tilde{Q}(s,\cdot) = \frac{1}{n} \sum_{i=1}^n 1 ( Y_{s,i} = \cdot )$. Furthermore, define
\begin{gather}
\overline{Q}(s,\cdot) = \begin{cases} \hat{Q}_{k_*}(s,\cdot) , & s \in U_{k_*} \\ \tilde{Q}(s,\cdot) , & s \in \mathcal{S} \setminus U_{k_*} \end{cases} , \quad \overline{\mu_s} = (1-\alpha) e_s^{\tr} ( I - \alpha \overline{Q})^{-1} , \quad \overline{v}(s) = \overline{\mu_s} c , \label{eqOverDefn} \\ 
\underline{Q}(s,\cdot) = \begin{cases} \hat{Q}_{k_*}(s,\cdot) , & s \in U_{k_*} \\ Q(s,\cdot)  , & s \in \mathcal{S} \setminus U_{k_*} \end{cases} , \quad \underline{\mu_s} = (1-\alpha) e_s^{\tr} ( I - \alpha \underline{Q})^{-1}  , \quad \underline{v}(s) = \underline{\mu_s} c , \label{eqUnderDefn} 
\end{gather}
where $k_* = \inf \{ k \in \Z_+ : \|r_k\|_{\infty} \leq \epsilon \}$ is the iteration at which Algorithm \ref{algBackwardEPE} terminates. Then
\begin{equation} \label{eqInvariantOverUnder}
\hat{v}_k(s) + \overline{\mu_s} r_k = \overline{v}(s), \quad \hat{v}_k(s) + \underline{\mu_s} r_k = \underline{v}(s) \quad \forall\ k \in \{0,\ldots,k_*\}, s \in \mathcal{S}\ a.s.
\end{equation}
\end{lemma}
\begin{proof}
Appendix \ref{proofLemInvariantUnknownQ}.
\end{proof}

Owing to the fact that \eqref{eqInvariantOverUnder} holds across iterations, we will refer to the identities in \eqref{eqInvariantOverUnder} as the \textit{$\overline{Q}$-invariant} and the \textit{$\underline{Q}$-invariant}, respectively. These invariants will be pivotal in the theorems to come; interestingly, though, only one invariant is useful for each theorem, while the other fails. This is due to technical issues discussed in Remarks \ref{remOverOrUnder_acc1}, \ref{remOverOrUnder_enc}, and \ref{remOverOrUnder_acc2} in the appendix. We also emphasize the offline estimate $\tilde{Q}$ is an analytical tool and does not affect our algorithm's sample complexity.

We turn to the first of the aforementioned theorems, an accuracy guarantee for \texttt{Backward- EPE}. Toward this end, note that $\overline{\mu_s}$ is a distribution over $\mathcal{S}$ and recall that $\|r_{k_*}\|_{\infty} \leq \epsilon$ by definition; thus, the $\overline{Q}$-invariant ensures that the ultimate estimate $\hat{v}_{k_*}(s)$ of $v(s)$ satisfies
\begin{equation}
| \hat{v}_{k_*}(s) - v(s) | \leq | \hat{v}_{k_*}(s) - \overline{v}(s) | + | \overline{v}(s) - v(s) | = \overline{\mu_s} r_{k_*} + | \overline{v}(s) - v(s) | \leq \epsilon + | \overline{v}(s) - v(s) | .
\end{equation}
For the remaining summand $| \overline{v}(s) - v(s) |$, recall $v$ and $\overline{v}$ are the value functions defined on $Q$ and an estimate of $Q$, respectively. Thus, if the estimate of $Q$ is sufficiently acccurate, this remaining summand will be small. This is made precise by the following theorem. We note that showing $\overline{v} \approx v$ with high probability is not immediate, because $\overline{v}$ is a biased estimate of $v$ in general; instead, the proof bounds $\| \overline{v}- v \|_{\infty}$ by a random variable more conducive to standard Chernoff bounds. We also note this $l_{\infty}$ guarantee matches the forward approach's guarantee from \cite{haskell2016empirical}.
\begin{theorem} \label{thmAccuracy}
Fix $\epsilon, \delta > 0$ and define
\begin{equation}
n^*(\epsilon,\delta) =  \frac{ 2 \| c \|_{\infty}^2 \alpha^2 }{ \epsilon^2(1-\alpha)^2}  \log \left( \frac{2 S}{\delta} \ceil*{ \frac{ \log ( 4 \| c \|_{\infty} / \epsilon ) }{ 1-\alpha } } \right) .
\end{equation}
Then assuming $n \geq n^*(\epsilon,\delta)$ in Algorithm \ref{algBackwardEPE}, $\P ( \| \hat{v}_{k_*} - v \|_{\infty} \geq 2 \epsilon ) \leq \delta$.
\end{theorem}
\begin{proof}
See Appendix \ref{proofThmAccuracy}.
\end{proof}

Theorem \ref{thmAccuracy} says that if we take $n \geq n^*(\epsilon,\delta)$ samples per state encountered, the estimate $\hat{v}_{k_*}$ produced by \texttt{Backward-EPE} will be $2\epsilon$-accurate. Since \texttt{Backward-EPE} encounters $|U_{k_*}|$ states by definition, the total number of samples needed to ensure $2\epsilon$-accuracy is $n^*(\epsilon,\delta) |U_{k_*}|$. Hence, our next goal is to bound $|U_{k_*}|$, in order to bound this overall complexity. By the backward exploration intuition discussed in Section \ref{secIntro}, we should expect a nontrivial bound $|U_{k_*}| = o(S)$ if the cost vector and supergraph are sufficiently sparse. However, even when both objects are maximally sparse, one can construct adversarial examples for which $U_{k_*} = \mathcal{S}$. For instance, suppose we restrict to $c$ having a single high-cost state and the supergraph to having the minimal number of edges $S$. Then taking $c = [1\ 0\ \ldots\ 0]$ and $A = 1_{S \times 1} e_1^{\tr}$ will satisfy this restriction, but will yield $U_{k_*} = \mathcal{S}$ (assuming $\epsilon < 1$). Note the key issue in this example (and, we suspect, in most adversarial examples) is the interaction between the cost vector and the supergraph; in particular, if high-cost states have high in-degrees, $|U_{k_*}|$ will be large (even if there are few high-cost states and few edges overall).

In light of this, our best hope for a nontrivial bound on $|U_{k_*}|$ is an average-case analysis; in particular, bounding $\E |U_{k_*}|$ while randomizing over the inputs of \texttt{Backward-EPE}. As it turns out, we only need to randomize over the cost vector (not the transition matrix). Roughly speaking, we will consider a random cost vector $C$ for which $\E C(s) = O ( \E \| C \|_1 / S )\ \forall\ s \in \mathcal{S}$, i.e.\ the expected cost of any given state does not dominate the average expected cost. For such cost vectors, the interaction between cost and in-degree discussed in the previous paragraph will ``average out'', and consequently the adversarial examples will not dominate in expectation. 

This intuition is formalized in the following theorem. Similar to Theorem \ref{thmAccuracy}, the proof exploits the $\overline{Q}$-invariant. Here the key observations are that  $\hat{v}_k(s) \leq \overline{v}(s)$ and that $\hat{v}_k(s)$ increases by at least $(1-\alpha) \epsilon$ at each $k$ for which $s_k = s$, which prevents certain states from being chosen as $s_k$ and thus (potentially) prevents their incoming supergraph neighbors from being encountered.

\begin{theorem} \label{thmEncountered}
Let $C$ be an $\R_+^S$-valued random vector s.t.\ $\E \|C\|_1 < \infty , \E C(s) \leq \beta \E \| C \|_1 / S = : \bar{c}$ for some absolute constant $\beta \in [1,\infty)$. Then if Algorithm \ref{algBackwardEPE} is initialized with cost vector $C$,
\begin{equation}
\E | U_{k_*} | \leq \frac{ S \bar{c} \bar{d} }{ \epsilon (1-\alpha) } ,
\end{equation}
where the expectation is with respect to $C$ and the randomness in Algorithm \ref{algBackwardEPE}.
\end{theorem}
\begin{proof}
See Appendix \ref{proofThmEncountered}.
\end{proof}


We now return to interpret our results and derive \texttt{Backward-EPE}'s overall sample complexity, which (we recall) is $n^*(\epsilon,\delta) |U_{k_*} |$. In the worst case, $|U_{k_*}| = \Omega(S)$, and thus the worst-case sample complexity for fixed $c$ is $O(S  n^*(\epsilon,\delta))$. Neglecting $\log \log$ factors and constants, ignoring $\log$ terms for quantities that have polynomial scaling (e.g.\ writing $\log (1 / (1-\alpha) ) / (1-\alpha)^2$ as simply $1/ (1-\alpha)^2$), and assuming $\alpha$ is either constant or grows to $1$, Theorem \ref{thmAccuracy} implies
\begin{equation}
S  n^*(\epsilon,\delta) = O \left(  S \log ( S / \delta ) \| c \|_{\infty}^2 \epsilon^{-2} (1-\alpha)^{-2} \right) .
\end{equation}
For comparison, the complexity of the forward approach is
\begin{equation}\label{eqStandAnyCase}
O \left(  S \log ( S / \delta ) \| c\|_{\infty}^2  \epsilon^{-2} (1-\alpha)^{-3}  \right) 
\end{equation}
(see Appendix \ref{appStandardApproach}). Thus, in the worst case \texttt{Backward-EPE} has similar complexity to that of the forward approach, with a slightly improved dependence on the discount factor $\alpha$. 

In the average case, however, the sample complexity of \texttt{Backward-EPE} can be dramatically better than the forward approach. In particular, Theorem \ref{thmEncountered} implies average-case sample complexity
\begin{equation}
\E |U_{k_*}| \times n^*(\epsilon,\delta) =  O \left( \frac{S \bar{c} \bar{d}}{ \epsilon(1-\alpha) }  \times \frac{  \log ( S / \delta ) \| C \|_{\infty}^2 }{ \epsilon^2 (1-\alpha)^2 } \right) = O \left( \frac{ \| C \|_1 \bar{d} \log ( S / \delta ) \| C \|_{\infty}^2 }{ \epsilon^3 (1-\alpha)^3 } \right) .
\end{equation}
(This argument is not precise, since $\|C\|_{\infty}$ is random in Theorem \ref{thmEncountered}; we return to address this shortly.) Thus, if $\alpha$, $\delta$, and $\| C \|_{\infty} / \epsilon$ are constants, \texttt{Backward-EPE} has average case complexity
\begin{equation}\label{eqBackEpeAvgCase}
O \left(  ( \| C \|_1 / \| C \|_{\infty}  ) \times \bar{d} \times \log S  \right) .
\end{equation}
Interestingly, \eqref{eqBackEpeAvgCase} exactly captures the intuition that backward exploration is efficient when the costs and supergraph are sufficiently sparse, since $\|C\|_1 / \|C\|_{\infty}$ and $\bar{d}$ quantify cost and supergraph sparsity, respectively. We also note that when $\alpha$, $\delta$, and $\| C \|_{\infty} / \epsilon$ are constants, the forward approach's complexity \eqref{eqStandAnyCase} becomes simply $O(S \log S)$. In the extreme case, $\|C\|_1 / \|C\|_{\infty}, \bar{d} = O(1)$ and \texttt{Backward-EPE} offers a dramatic reduction in sample complexity; namely, by a factor of $S$.

Though this average-case argument is not precise, we can make it rigorous with further assumptions on $C$. For example, the following corollary considers random binary cost vectors with $H$ nonzero entries. Such cost vectors could arise, for example, in MDP models of games, where states corresponding to losing configurations of the game have unit cost and other states have zero cost.
\begin{corollary} \label{corHnonzeros}
Fix $H \in \mathcal{S}$ and define $\mathcal{C}_H = \{ \sum_{s=1}^S a_s e_s : a_s \in \{0,1\}\ \forall\ s \in \mathcal{S}, \sum_{s=1}^S a_s = H \}$ to be the set of binary vectors with $H$ nonzero entries. Assume the cost vector $C$ is chosen uniformly at random from $\mathcal{C}_H$ and $\alpha,\delta,\epsilon$ are constants. Then to guarantee $\P ( \| \hat{v}_{k_*} - v \|_{\infty} \geq 2 \epsilon ) \leq \delta$, \texttt{Backward-EPE} requires $O( \min \{ H \bar{d} , S \}  \log S )$ samples in expectation.
\end{corollary}
\begin{proof}
See Appendix \ref{proofCorHnonzeros}.
\end{proof}


To conclude this section and illustrate our analysis, we present empirical results in Figure \ref{figBackwardNumerical}. Here we generate random problem instances $Q,c$ in a manner that yields three different cases of the complexity factor $\bar{d} \|C\|_1 / \|C\|_{\infty}$ identified above; roughly, $\Theta(1)$, $\Theta(\sqrt{S})$, and $\Theta(S)$ (left). In all cases, the sample complexity of \texttt{Backward-EPE} decays relative to that of the forward approach, suggesting sublinear complexity (middle). Moreover, the different scalings of $\bar{d} \|C\|_1 / \|C\|_{\infty}$ reflect in different rates of decay in relative complexity, suggesting $\bar{d} \|C\|_1 / \|C\|_{\infty}$ indeed determines sample complexity. We also note algorithmic parameters are chosen to ensure both algorithms yield similar $l_{\infty}$ error (right). Error bars show standard deviation across problem instances. Further details regarding the experimental setup can be found in Appendix \ref{appExperiment}.

\begin{figure}
\centering
\includegraphics[height=2.25in]{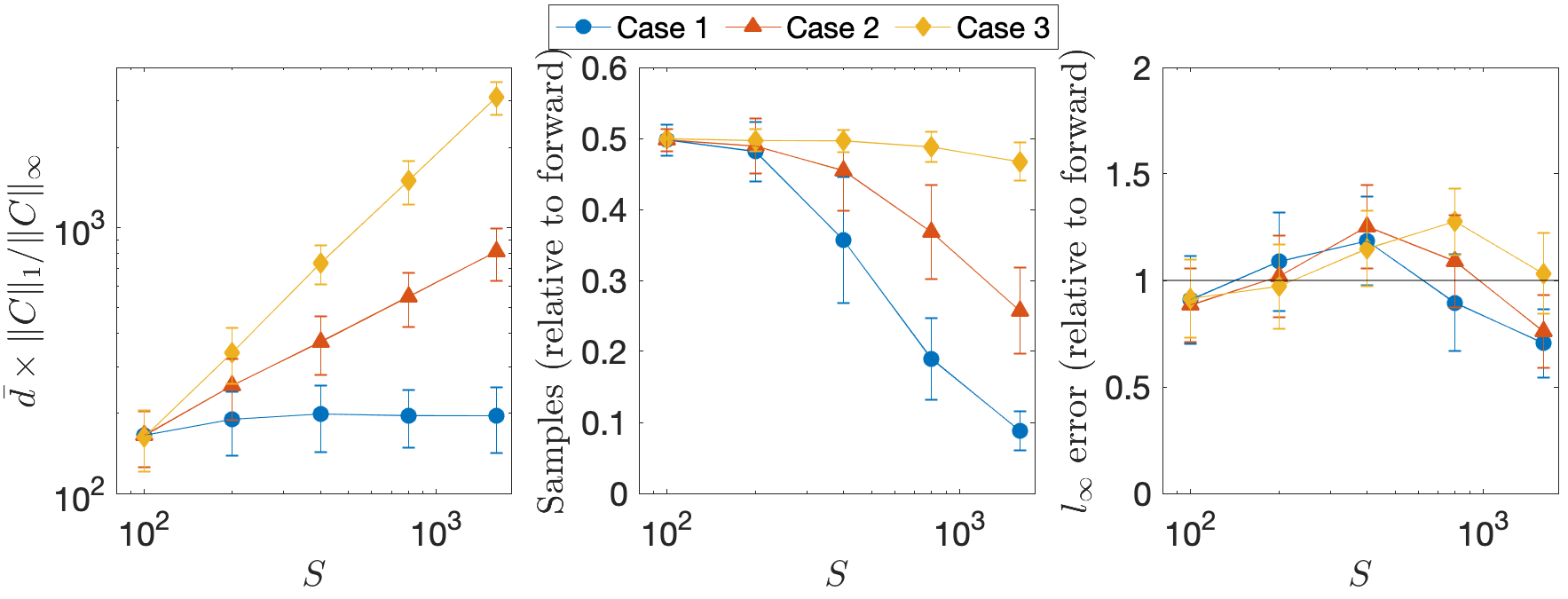}
\caption{Numerical illustration of \texttt{Backward-EPE}} \label{figBackwardNumerical}
\end{figure}

\section{Bidirectional empirical policy evaluation} \label{secBidirectional}

Our second algorithm is called \texttt{Bidirectional-EPE} and is inspired by the \texttt{Bidirectional- PPR} PageRank estimator from \cite{lofgren2016personalized} (see Appendix \ref{appComparison} for further discussion of this PageRank estimator). As will be seen, this algorithm is conducive to a stronger accuracy guarantee; namely, a (pseudo)-relative error guarantee. The utility of such a guarantee is that the resulting estimates tend to better preserve the ordering of the actual value function when compared to an $l_{\infty}$ guarantee. Preserving this ordering is important in the problem of finding good policies; e.g.\ in the greedy update of policy iteration mentioned in Section \ref{secIntro}.

As its name suggests, \texttt{Bidirectional-EPE} proceeds in two stages: it first conducts backward exploration using \texttt{Backward-EPE}, then improves the resulting estimate via forward exploration. The analysis of this bidirectional approach relies on the $\underline{Q}$-invariant \eqref{eqInvariantOverUnder}. Similar to Theorem \ref{thmAccuracy}, we can make $| \underline{v}(s) - v(s) |$ small by taking $n$ large in \texttt{Backward-EPE}; when this holds, we have
\begin{equation}\label{eqUnderInvariantApprox}
v(s) \approx \underline{v}(s) = \hat{v}_{k_*}(s) + \underline{\mu_s} r_{k_*} . 
\end{equation}
Since $\underline{\mu_s}$ is a probability distribution over $\mathcal{S}$, the residual term in \eqref{eqUnderInvariantApprox} satisfies
\begin{equation}
\underline{\mu_s} r_{k_*} = \E_{Z_s \sim \underline{\mu_s}} r_{k_*}(Z_s) \approx \frac{1}{n_F} \sum_{i=1}^{n_F} r_{k_*} ( Z_{s,i} ) ,
\end{equation}
where in the approximate equality $\{ Z_{s,i} \}_{i=1}^{n_F}$ are distributed as $\underline{\mu_s}$ and $n_F$ is large. Hence, by \eqref{eqUnderInvariantApprox},
\begin{equation}\label{eqBidirectionalApproxEq}
v(s) \approx \hat{v}_{k_*}(s) + \frac{1}{n_F} \sum_{i=1}^{n_F} r_{k_*} ( Z_{s,i} ) .
\end{equation}
Intuitively, the right side of \eqref{eqBidirectionalApproxEq} is a more accurate estimate of $v(s)$ than $\hat{v}_{k_*}(s)$ alone; the only remaining question is how to generate $\{ Z_{s,i} \}_{i=1}^{n_F}$. This can indeed be done in our model; namely, by generating $\textrm{Geometric}(1-\alpha)$-length trajectories on $\underline{Q}$. More specifically, given $\underline{Q}$, we first  generate a $\textrm{Geometric}(1-\alpha)$ random variable $L_{s,i}$ and set $Z_{s,i}^{0} = s$; we then sample $Z_{s,i}^{t}$ from $\underline{Q}(Z_{s,i}^{t-1},\cdot)$ for each $t \in [L_{s,i}]$; and finally we set $Z_{s,i} = Z_{s,i}^{L_{s,i}}$. Then conditioned on $\underline{Q}$, $Z_{s,i}$ is distributed as $\underline{\mu_s}$. To see why, let $\P^{\underline{Q}}$ denote probability conditioned on $\underline{Q}$ and observe
\begin{equation}
\P^{\underline{Q}} ( Z_{s,i} = s' ) = \sum_{t=0}^{\infty} \P^{\underline{Q}} ( Z_{s,i} = s' | L_{s,i} = t ) \P^{\underline{Q}} ( L_{s,i} = t ) = \sum_{t=0}^{\infty} \underline{Q}^t(s,s') (1-\alpha) \alpha^t = \underline{\mu_s}(s') .
\end{equation}
Thus, sampling from $\underline{\mu_s}$ amounts sampling from $\underline{Q}(s,\cdot)$. To do so, we either sample from $Q(s,\cdot)$ (if $s \notin U_{k_*}$) or from $\hat{Q}_{k_*}(s,\cdot)$ (if $s \in U_{k_*}$); the former is exactly what was done in \texttt{Backward-EPE}, and the latter can be done after running \texttt{Backward-EPE}. Put differently, to generate $Z_{s,i}$ we sample from $Q(s,\cdot)$ unless we have already sampled from $Q(s,\cdot)$ during \texttt{Backward-EPE}, in which case we sample from the empirical estimate $\hat{Q}_{k_*}(s,\cdot)$ obtained during \texttt{Backward-EPE}.

The \texttt{Bidirectional-EPE} algorithm is formally defined in Algorithm \ref{algBidirectionalEPE}. As above, write $n_F$ for the per-state forward trajectory count; we also write $n_B$ for the per-state sample count in the \texttt{Backward-EPE} subrountine. We denote the ultimate estimate of $v$ by $\hat{v}_{BD}$. 

\begin{algorithm}
\caption{ \texttt{Bidirectional-EPE} } \label{algBidirectionalEPE}
\KwIn{Sampler for transition matrix $Q$; cost vector $c$; discount factor $\alpha$; supergraph in-neighbors $\{N_{in}(s)\}_{s=1}^S$; termination parameter $\epsilon$; per-state backward, forward sample counts $n_B$, $n_F$}

Run \texttt{Backward-EPE} (Algorithm \ref{algBackwardEPE}) with inputs $Q$ sampler, $c$, $\alpha$, $\{N_{in}(s)\}_{s=1}^S$, $\epsilon$, $n_B$

Let $\hat{v}_{k_*}$, $r_{k_*}$, $U_{k_*}$, $\hat{Q}_{k_*}$ be estimate vector, residual vector, encountered states, and $Q$ estimate at termination of \texttt{Backward-EPE}, and define $\underline{Q}$ in \eqref{eqUnderDefn}

\For{$s \in \mathcal{S}$}{

	Generate samples $\{ Z_{s,i} \}_{i=1}^n$ from $\underline{\mu_s}$, set $\hat{v}_{BD}(s) = \hat{v}_{k_*}(s) + \frac{1}{n_F} \sum_{i=1}^{n_F} r_{k_*} ( Z_{s,i} )$

}
\KwOut{Estimate $\hat{v}_{BD}$ of $v = (1-\alpha) \sum_{t=0}^{\infty} \alpha^t Q^t c$}
\end{algorithm}


As alluded to above, \texttt{Bidirectional-EPE} is conducive to a pseudo-relative error guarantee. In particular, given relative error tolerance $\epsilon_{rel} \in (0,1)$ and absolute tolerance $\epsilon_{abs} > 0$, Theorem \ref{thmAccuracy2} shows that with high probability, the estimate $\hat{v}_{BD}$ satisfies
\begin{equation}
( 1 - \epsilon_{rel} ) v(s) - \epsilon_{abs} \leq \hat{v}_{BD}(s) \leq (1+\epsilon_{rel}) v(s) + \epsilon_{abs}\ \forall\ s \in \mathcal{S} .
\end{equation}
Thus, \texttt{Bidirectional-EPE} permits a relative-plus-absolute accuracy guarantee. (Note that since $v(s)$ can be arbitrarily small in general, we should not expect a relative error guarantee for all states.) This guarantee is formalized in the next theorem. As suggested by \eqref{eqUnderInvariantApprox}-\eqref{eqBidirectionalApproxEq}, the proof first shows $\underline{v} \approx v$ for $n_B$ large; conditioned on $\underline{v} \approx v$, we then show $\E_{Z_s \sim \underline{\mu_s}} r_{k_*}(Z_s) \approx \frac{1}{n_F} \sum_{i=1}^{n_F} r_{k_*} ( Z_{s,i} )$ for $n_F$ large, using separate Chernoff bounds for two cases of $\E_{Z_s \sim \underline{\mu_s}} r_{k_*}(Z_s)$.

\begin{theorem} \label{thmAccuracy2}
Fix $\epsilon_{rel} \in (0,1)$ and $\epsilon_{abs}, \delta > 0$, and define
\begin{gather}
n^*_F(\epsilon_{rel}, \epsilon_{abs}, \delta) = \frac{ 324 \epsilon \log ( 4 S / \delta ) }{ \epsilon_{rel}^2 \epsilon_{abs} } , \\
n^*_B(\epsilon_{rel}, \epsilon_{abs}, \delta) = \frac{3 \log ( 4 S^2 / \delta ) }{ ( \log ( 1 + \epsilon_{rel} / 2 ) )^2 \min_{i,j \in \mathcal{S}: Q(i,j) > 0} Q(i,j) }  \ceil*{ \frac{\log ( 2 \| c \|_{\infty} / \epsilon_{abs} )}{ (1-\alpha) } }^2 .
\end{gather}
Then assuming $n_F \geq n^*_F(\epsilon_{rel}, \epsilon_{abs}, \delta)$ and $n_B \geq n^*_B(\epsilon_{rel}, \epsilon_{abs}, \delta)$ in Algorithm \ref{algBidirectionalEPE}, we have
\begin{equation}\label{eqAccuracyGuarantee2}
\P ( \cup_{s=1}^S \{ | \hat{v}_{BD}(s) - v(s) | > \epsilon_{rel} v(s) + \epsilon_{abs} \} ) \leq \delta .
\end{equation}
\end{theorem}
\begin{proof}
See Appendix \ref{proofThmAccuracy2}.
\end{proof}


We next discuss Theorem \ref{thmAccuracy2}. To simplify notation, we restrict to the setting of Corollary \ref{corHnonzeros}; however, the key insights extend to the more general setting of Theorem \ref{thmEncountered}. Also, we assume the relative error tolerance $\epsilon_{rel}$, the discount factor $\alpha$, and inaccuracy probability $\delta$ are constants independent of $S$. Finally, we note Theorem \ref{thmAccuracy2} holds for random $C$; see Remark \ref{remAccProof2Crand} in Appendix \ref{proofThmAccuracy2}.

We begin by deriving expressions for the asymptotic sample complexity of \texttt{Bidirectional- EPE} in the setting of Corollary \ref{corHnonzeros}. For the backward exploration stage (i.e.\ the \texttt{Backward-EPE} subroutine), we require per-state sample complexity $n^*_B(\epsilon_{rel}, \epsilon_{abs}, \delta)$; note this is deterministic since $\|C\|_{\infty} = 1$ pointwise in Corollary \ref{corHnonzeros}. Thus, the average-case sample complexity is (by Corollary \ref{corHnonzeros}),
\begin{equation}\label{eqBidirCompBack}
 n^*_B(\epsilon_{rel}, \epsilon_{abs}, \delta) \E  |U_{k_*} | = O \left( \frac{ \log(S) \log ( 1 / \epsilon_{abs} ) }{  \min_{i,j \in \mathcal{S}: Q(i,j) > 0} Q(i,j) } \times \frac{H \bar{d}}{\epsilon} \right).
\end{equation} 
For the forward exploration stage, we require $n_F^*(\epsilon_{rel},\epsilon_{abs},\delta) = O(\epsilon \log(S) / \epsilon_{abs} )$ trajectories of expected length $\alpha/(1-\alpha)$ for each of $S$ states. We are assuming $\alpha$ is a constant, and thus the expected forward complexity is simply $O( \epsilon S \log S / \epsilon_{abs} )$. Combined with \eqref{eqBidirCompBack}, and writing $K_{BD}$ for the overall expected sample of \texttt{Bidirectional-EPE} in the setting of Corollary \ref{corHnonzeros},
\begin{equation}\label{eqBidirCompCorSetting}
K_{BD} = O \left( \frac{ H \bar{d} \log(S) \log ( 1 / \epsilon_{abs} ) }{ \epsilon  \min_{i,j \in \mathcal{S}: Q(i,j) > 0} Q(i,j) } + \frac{\epsilon S \log S}{\epsilon_{abs} }  \right) .
\end{equation}
Here the termination parameter $\epsilon$ for the \texttt{Backward-EPE} subroutine is a free parameter that can be chosen to minimize the overall sample complexity. For example,
\begin{equation}\label{eqBidirCompCorSetting_eps}
\epsilon = \Theta \left( \sqrt{ \frac{ H \bar{d} \epsilon_{abs} }{ S \min_{i,j \in \mathcal{S}: Q(i,j) > 0} Q(i,j) } } \right)  \Rightarrow  K_{BD} = O \left(  \sqrt{    \frac{  S H \bar{d}   }{  \epsilon_{abs} \min_{i,j \in \mathcal{S}: Q(i,j) > 0} Q(i,j)  }  }  \log S \right) ,
\end{equation}
where for simplicity we wrote $\log(1/\epsilon_{abs}) / \sqrt{\epsilon_{abs}}$ as simply $1/ \sqrt{\epsilon_{abs}}$ (note this choice of $\epsilon$ minimizes \eqref{eqBidirCompCorSetting} if we also ignore the $\log(1/\epsilon_{abs})$ term in that expression). To interpret \eqref{eqBidirCompCorSetting_eps}, we consider a specific choice of $\epsilon_{abs}$. To motivate this, we first observe that in the setting of Corollary \ref{corHnonzeros},
\begin{equation}
\E v = (1-\alpha) \sum_{t=0} \alpha^t Q^t \times \E C = (1-\alpha) \sum_{t=0} \alpha^t Q^t \times \frac{H}{S} 1_{S \times 1} = \frac{H}{S} 1_{S \times 1} ,
\end{equation}
i.e.\ the ``typical'' value is $H/S$. It is thus sensible to choose $\epsilon_{abs} = \Theta(H/S)$, so that we obtain a relative guarantee for above-typical values and settle for the absolute guarantee for below-typical values. Substituting into \eqref{eqBidirCompCorSetting_eps}, we conclude that \texttt{Bidirectional-EPE} requires
\begin{equation}\label{eqBidirCompCorSetting_fin}
K_{BD} = O \left(  \sqrt{ \frac{\bar{d}   }{  \min_{i,j \in \mathcal{S}: Q(i,j) > 0} Q(i,j) } } S \log S\right) 
\end{equation}
samples in order to guarantee \eqref{eqAccuracyGuarantee2} in the setting of Corollary \ref{corHnonzeros}.

It is interesting to compare \texttt{Bidirectional-EPE} to a plug-in estimator that lends itself to the same accuracy guarantee. For this plug-in estimator, we simply estimate $v$ as $(1-\alpha) \sum_{t=0}^{\infty} \alpha^t \tilde{Q}^t C$, where $\tilde{Q}(s,\cdot) = \frac{1}{n} \sum_{i=1}^n 1 ( Y_{s,i} = \cdot )$ with $Y_{s,i} \sim Q(s,\cdot)$ as in Lemma \ref{lemInvariantUnknownQ}. Then by the same argument following \eqref{eqRelErrPassToTilde} in the proof of Theorem \ref{thmAccuracy2}, the plug-in estimate will satisfy the guarantee \eqref{eqAccuracyGuarantee2} whenever $n \geq n_B^*(\epsilon_{rel},\epsilon_{abs},\delta)$. Consequently, the sample complexity of the plug-in estimator is, under the assumptions leading to \eqref{eqBidirCompCorSetting_fin},
\begin{equation}\label{eqPlugInComp}
S n_B^*(\epsilon_{rel},\epsilon_{abs},\delta) = O \left(   \frac{ S \log S }{ \min_{i,j \in \mathcal{S}: Q(i,j) > 0} Q(i,j) }  \right) .
\end{equation}
Comparing \eqref{eqBidirCompCorSetting_fin} and \eqref{eqPlugInComp}, we see \texttt{Bidirectional-EPE} is more efficient than the plug-in whenever $\bar{d} \leq 1 / \min_{i,j \in \mathcal{S}: Q(i,j) > 0} Q(i,j)$. To interpret this inequality, first suppose the supergraph is precisely the graph induced by $Q$, i.e.\ $A(s,s') = 0 \Leftrightarrow Q(s,s') = 0$. Then for any $s \in \mathcal{S}$, we have
\begin{equation}
\sum_{s' \in \mathcal{S}} A(s,s') =\sum_{s' \in \mathcal{S} : Q(s,s') > 0} \frac{Q(s,s')}{Q(s,s')} \leq \frac{ \sum_{s' \in \mathcal{S} : Q(s,s') > 0} Q(s,s') }{ \min_{i,j \in \mathcal{S} : Q(i,j) > 0} Q(i,j) } = \frac{1}{ \min_{i,j \in \mathcal{S} : Q(i,j) > 0} Q(i,j) } ,
\end{equation}
so $\bar{d} \leq 1 / \min_{i,j \in \mathcal{S}: Q(i,j) > 0} Q(i,j)$ indeed holds. More generally, this suggests that the complexity of \texttt{Bidirectional-EPE} is order-wise similar to that of the plug-in method whenever degrees in the supergraph and induced graph are order-wise similar. If most positive transition probabilities dominate the minimum probability, then $\bar{d} = o ( 1 / \min_{i,j \in \mathcal{S}: Q(i,j) > 0} Q(i,j) )$, in which case \texttt{Bidirectional-EPE} is strictly better asymptotically.

Generally, it is difficult to compare the sample complexity \eqref{eqBidirCompCorSetting_fin} to the bounds derived in Section \ref{secBackward} analytically, owing to the different error guarantees. Thus, we present an empirical comparison in Figure \ref{figBidirectionalNumerical}. Here we simulate all three algorithms using the case $\bar{d} \|C\|_1 / \|C\|_{\infty} \approx \Theta(1)$ from Figure \ref{figBackwardNumerical}. We choose algorithmic parameters so that all algorithms maintain average relative error $\frac{1}{S} \sum_{s=1}^S \frac{ | \hat{v}(s) - v(s) | }{ v(s) } \approx 25\%$ across $S$ (right). For these parameters, the sample complexities of the forward approach and \texttt{Backward-EPE} scale like $S^2$ (obtained via linear fits on a log-log scale, left). In contrast, the complexity of \texttt{Bidirectional-EPE} scales like $S^{1.7}$, suggesting a subquadratic sample complexity. Thus, as discussed above, \texttt{Bidirectional-EPE} appears more sample-efficient if one aims to maintain constant relative error.

\begin{figure}
\centering
\includegraphics[height=2.25in]{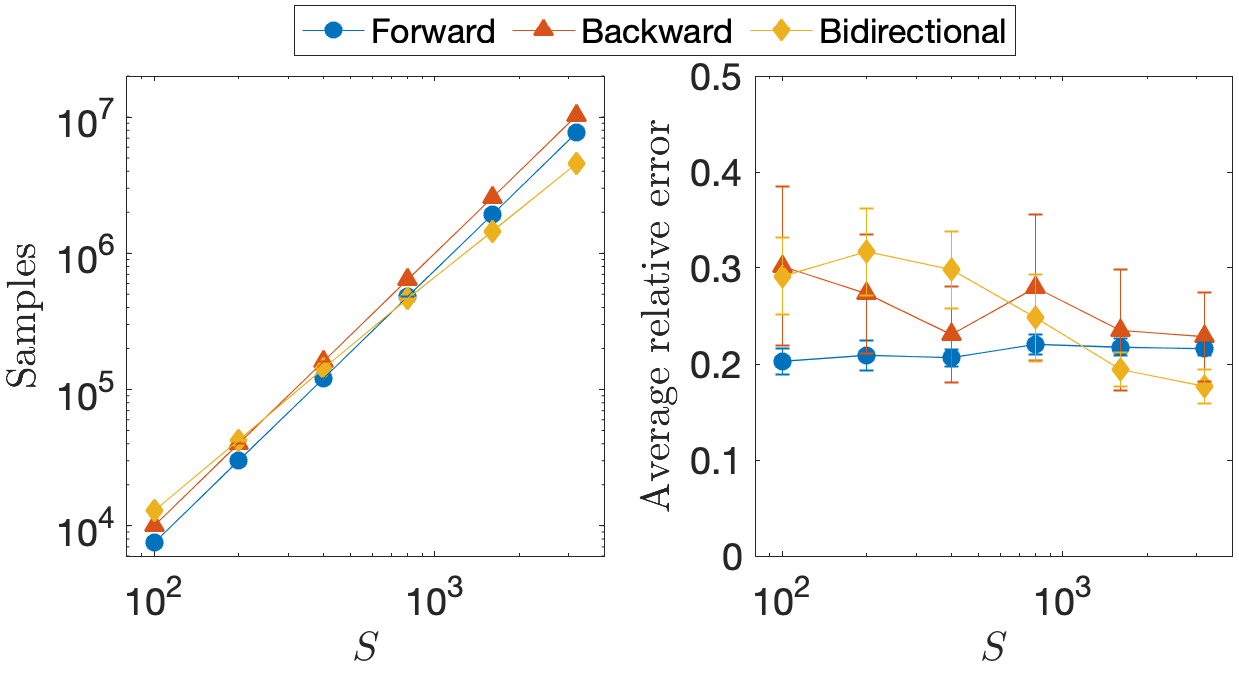}
\caption{Numerical illustration of \texttt{Bidirectional-EPE}} \label{figBidirectionalNumerical}
\end{figure}

\section{Future directions} \label{secFuture}

In this work, we adapted the PageRank estimators from \cite{andersen2008local,lofgren2016personalized} to EPE. However, the PageRank literature contains many other algorithms either explicitly or conceptually related to these estimators, see e.g.\ \cite{jeh2003scaling, andersen2006local,berkhin2006bookmark,wang2017fora,vial2019role,vial2019structural}. Each of these algorithms rely on analyses similar to that of \cite{andersen2008local,lofgren2016personalized}, which we extended to the EPE setting in this work. Thus, while we have focused on two specific algorithms in this paper, our analysis should be viewed as an example of how to extend a larger family of algorithms to EPE.

Another extension of this work is devising backward and bidirectional exploration-based EPE algorithms for the finite horizon cumulative cost value function
\begin{equation}
v(s) = \E \left[ \sum_{t=0}^T c ( Z_t ) \middle| Z_0 = s \right] = \sum_{t=0}^T Q^t(s,\cdot) c .
\end{equation}
Here one aims to estimate multi-step transition distributions of the form $Q^t(s,\cdot)$. Though our algorithms do not immediately apply, relevant analogues of \texttt{Approx-Contributions} exist in the case where $Q$ is known. In particular, \cite{banerjee2015fast} provides an algorithm to estimate $Q^t(s,\cdot)$ when $Q$ is known. The algorithm is analogous to \texttt{Approx-Contributions} in that it explores backward from high-cost states. Moreover, \cite{banerjee2015fast} provides a bidirectional variant. Both of these algorithms could be adapted to EPE using our approach; this would yield analogues of \texttt{Backward-EPE} and \texttt{Bidirectional-EPE} for finite horizons.

As mentioned in Section \ref{secBackward}, an alternative of \texttt{Backward-EPE} would take independent samples from $Q(s,\cdot)$ for each $s \in N_{in}(s_k)$ and at each iteration $k$, rather than only sampling from $Q(s,\cdot)$ when we first encounter $s$ as in \texttt{Backward-EPE}. This alternative scheme is formally defined in Appendix \ref{appResampling}. An interesting property is that, while the invariants of Lemma \ref{lemInvariantUnknownQ} fail, a related error process is a zero-mean martingale (see Appendix \ref{appResampling}), and thus the ultimate estimate is unbiased. Analytically, this is an advantage over \texttt{Backward-EPE}, where the $\overline{Q}$- and $\underline{Q}$-invariants hold but the corresponding value functions $\overline{v}$ and $\underline{v}$ are biased estimates of $v$. The disadvantage of this alternative approach is that it may sample many times from each row of $Q$, and thus the overall sample complexity may exceed that of the forward approach. Put differently, \texttt{Backward-EPE} is conservative in the sense that it performs no worse than the forward approach in the worst case (see Section \ref{secBackward}), but it sacrifices desirable properties that could perhaps improve performance in other cases. A useful avenue for future work would thus be to investigate this tradeoff.

\bibliography{references}

\begin{thebibliography}{16}
\providecommand{\natexlab}[1]{#1}
\providecommand{\url}[1]{\texttt{#1}}
\expandafter\ifx\csname urlstyle\endcsname\relax
  \providecommand{\doi}[1]{doi: #1}\else
  \providecommand{\doi}{doi: \begingroup \urlstyle{rm}\Url}\fi

\bibitem[Andersen et~al.(2006)Andersen, Chung, and Lang]{andersen2006local}
Reid Andersen, Fan Chung, and Kevin Lang.
\newblock Local graph partitioning using {PageRank} vectors.
\newblock In \emph{2006 IEEE Symposium on Foundations of Computer Science},
  pages 475--486. IEEE, 2006.

\bibitem[Andersen et~al.(2008)Andersen, Borgs, Chayes, Hopcroft, Mirrokni, and
  Teng]{andersen2008local}
Reid Andersen, Christian Borgs, Jennifer Chayes, John Hopcroft, Vahab Mirrokni,
  and Shang-Hua Teng.
\newblock Local computation of {PageRank} contributions.
\newblock \emph{Internet Mathematics}, 5\penalty0 (1-2):\penalty0 23--45, 2008.

\bibitem[Banerjee and Lofgren(2015)]{banerjee2015fast}
Siddhartha Banerjee and Peter Lofgren.
\newblock Fast bidirectional probability estimation in {Markov} models.
\newblock In \emph{Advances in Neural Information Processing Systems}, pages
  1423--1431, 2015.

\bibitem[Berkhin(2006)]{berkhin2006bookmark}
Pavel Berkhin.
\newblock Bookmark-coloring algorithm for personalized pagerank computing.
\newblock \emph{Internet Mathematics}, 3\penalty0 (1):\penalty0 41--62, 2006.

\bibitem[Dubhashi and Panconesi(2009)]{dubhashi2009concentration}
Devdatt~P Dubhashi and Alessandro Panconesi.
\newblock \emph{Concentration of measure for the analysis of randomized
  algorithms}.
\newblock Cambridge University Press, 2009.

\bibitem[Edwards et~al.(2018)Edwards, Downs, and Davidson]{edwards2018forward}
Ashley~D Edwards, Laura Downs, and James~C Davidson.
\newblock Forward-backward reinforcement learning.
\newblock \emph{arXiv preprint arXiv:1803.10227}, 2018.

\bibitem[Florensa et~al.(2017)Florensa, Held, Wulfmeier, Zhang, and
  Abbeel]{florensa2017reverse}
Carlos Florensa, David Held, Markus Wulfmeier, Michael Zhang, and Pieter
  Abbeel.
\newblock Reverse curriculum generation for reinforcement learning.
\newblock In \emph{Conference on Robot Learning}, pages 482--495, 2017.

\bibitem[Goyal et~al.(2018)Goyal, Brakel, Fedus, Singhal, Lillicrap, Levine,
  Larochelle, and Bengio]{goyal2018recall}
Anirudh Goyal, Philemon Brakel, William Fedus, Soumye Singhal, Timothy
  Lillicrap, Sergey Levine, Hugo Larochelle, and Yoshua Bengio.
\newblock Recall traces: Backtracking models for efficient reinforcement
  learning.
\newblock \emph{arXiv preprint arXiv:1804.00379}, 2018.

\bibitem[Haskell et~al.(2016)Haskell, Jain, and Kalathil]{haskell2016empirical}
William~B Haskell, Rahul Jain, and Dileep Kalathil.
\newblock Empirical dynamic programming.
\newblock \emph{Mathematics of Operations Research}, 41\penalty0 (2):\penalty0
  402--429, 2016.

\bibitem[Jeh and Widom(2003)]{jeh2003scaling}
Glen Jeh and Jennifer Widom.
\newblock Scaling personalized web search.
\newblock In \emph{Proceedings of the 12th International Conference on World
  Wide Web}, pages 271--279. ACM, 2003.

\bibitem[Lofgren and Goel(2013)]{lofgren2013personalized}
Peter Lofgren and Ashish Goel.
\newblock Personalized {PageRank} to a target node.
\newblock \emph{arXiv preprint arXiv:1304.4658}, 2013.

\bibitem[Lofgren et~al.(2016)Lofgren, Banerjee, and
  Goel]{lofgren2016personalized}
Peter Lofgren, Siddhartha Banerjee, and Ashish Goel.
\newblock Personalized {PageRank} estimation and search: {A} bidirectional
  approach.
\newblock In \emph{Proceedings of the Ninth ACM International Conference on Web
  Search and Data Mining}, pages 163--172. ACM, 2016.

\bibitem[Page et~al.(1999)Page, Brin, Motwani, and Winograd]{page1999pagerank}
Lawrence Page, Sergey Brin, Rajeev Motwani, and Terry Winograd.
\newblock The {PageRank} citation ranking: {B}ringing order to the web.
\newblock Technical report, Stanford InfoLab, 1999.

\bibitem[Vial and Subramanian(2019{\natexlab{a}})]{vial2019role}
Daniel Vial and Vijay Subramanian.
\newblock On the role of clustering in {Personalized PageRank} estimation.
\newblock \emph{ACM Transactions on Modeling and Performance Evaluation of
  Computing Systems (TOMPECS)}, 4\penalty0 (4):\penalty0 21,
  2019{\natexlab{a}}.

\bibitem[Vial and Subramanian(2019{\natexlab{b}})]{vial2019structural}
Daniel Vial and Vijay Subramanian.
\newblock A structural result for {Personalized PageRank} and its algorithmic
  consequences.
\newblock \emph{Proceedings of the ACM on Measurement and Analysis of Computing
  Systems}, 3\penalty0 (2):\penalty0 25, 2019{\natexlab{b}}.

\bibitem[Wang et~al.(2017)Wang, Yang, Xiao, Wei, and Yang]{wang2017fora}
Sibo Wang, Renchi Yang, Xiaokui Xiao, Zhewei Wei, and Yin Yang.
\newblock {FORA}: {S}imple and effective approximate single-source
  {Personalized PageRank}.
\newblock In \emph{Proceedings of the 23rd ACM SIGKDD International Conference
  on Knowledge Discovery and Data Mining}, pages 505--514. ACM, 2017.

\end{thebibliography}

\newpage \appendix

\section{Existing PageRank estimators} \label{appComparison}

The \texttt{Approx-Contributions} algorithm is defined in Algorithm \ref{algApproxCont}. As mentioned in Section \ref{secBackward}, \texttt{Backward-EPE} naturally generalizes this algorithm by initializing the residual vector as $c$ (instead of restricting to the case $e_{s_*}$) and by replacing $Q(s,s_k)$ with an empirical estimate in the iterative update. Also as mentioned in Section \ref{secBackward}, Lemma \ref{lemInvariantUnknownQ} is an analogue of Lemma 1 in \cite{andersen2008local}, which states
\begin{equation}\label{eqInvariantKnownQ}
\hat{v}_k(s) + \mu_s r_k = v(s)\ \forall\ k \in \{0,1,\ldots\} , s \in \mathcal{S} ,
\end{equation}
where $\mu_s = (1-\alpha) e_s^{\tr} ( I - \alpha Q)^{-1}$. The proof of Lemma \ref{lemInvariantUnknownQ} resembles that of Lemma 1 in \cite{andersen2008local} but addresses new technical issues that arise in the case of unknown $Q$; see Remark \ref{remInvariantProofKnownQ}. Similarly, Theorem \ref{thmEncountered} resembles existing \texttt{Approx-Contributions} computational complexity analyses while addressing new technical issues; see Remark \ref{remCompCompKnownQ}.

\begin{algorithm}
\caption{\texttt{Approx-Contributions} (from \cite{andersen2008local})} \label{algApproxCont}
\KwIn{Transition matrix $Q$; cost vector $e_{s^*}$; discount factor $\alpha$; termination parameter $\epsilon$}

$k = 0$, $\hat{v}_k = 0_{S \times 1}$, $r_k = e_{s^*}$

\While{$\|r_k\|_{\infty} > \epsilon$}{

	$k \leftarrow k+1$, $s_k \sim \argmax_{s \in \mathcal{S}} r_{k-1}(s)$ uniformly

	\For{$s \in \mathcal{S}$}{

		\lIf{$s = s_k$}{$\hat{v}_k(s) = \hat{v}_{k-1}(s) + (1-\alpha) r_{k-1}(s)$, $r_k(s) = \alpha Q(s,s_k) r_{k-1}(s_k)$} 

		\lElse{$\hat{v}_k(s) = \hat{v}_{k-1}(s)$, $r_k(s) = r_{k-1}(s) + \alpha Q(s,s_k) r_{k-1}(s_k)$}

	}
} 
\KwOut{Estimate $\hat{v}_{k}$ of $v = (1-\alpha) \sum_{t=0}^{\infty} \alpha^t Q^t e_{s_*}$}
\end{algorithm}

As a historical note, the restriction to $c = e_{s_*}$ arose because the original intent of \texttt{Approx- Contributions} was to estimate the $s^*$-th column of $(1-\alpha) (I-\alpha Q)^{-1}$. The column sums of this matrix are called PageRank scores and serve as a network centrality measure in the network science literature. Estimating the $s^*$-th column allows one to approximate how much each node in the network contributes to $s^*$'s PageRank score (hence the name \texttt{Approx-Contributions}).

As mentioned in Section \ref{secBidirectional}, \texttt{Bidirectional-EPE} adapts \texttt{Bidirectional-PPR} from \cite{lofgren2016personalized} to EPE. The latter algorithm first runs \texttt{Approx-Contributions}, then estimates the unknown residual $\mu_s r_{k_*} = \E_{Z_s \sim \mu_s} r_{k_*} ( Z_s )$ in \eqref{eqInvariantKnownQ} with $\textrm{Geometric}(1-\alpha)$-length trajectories on $Q$. Thus, \texttt{Bidirectional-EPE} adapts this by replacing \texttt{Approx-Contributions} with its analogue \texttt{Backward-EPE}, and by sampling trajectories on $\underline{Q}$ instead of $Q$ (since \eqref{eqInvariantKnownQ} fails for \texttt{Backward-EPE} but the $\overline{Q}$-invariant holds).

\section{Proof of Lemma \ref{lemInvariantUnknownQ}} \label{proofLemInvariantUnknownQ}

We state and prove a slightly more general result.

\begin{lemma} \label{lemInvariantUnknownQgeneral}
Let $\mathcal{P} = \{ B \in \R_+^{S \times S} : \sum_{s'=1}^S B(s,s') = 1\ \forall\ s \in \mathcal{S} \}$ denote the set of $S \times S$ row stochastic matrices, and let $P$ be a random matrix taking values in $\mathcal{P}$ and satisfying the following:
\begin{equation}
P(s,\cdot) = \hat{Q}_{k_*}(s,\cdot)\ \forall\ s \in U_{k_*}\ a.s., \quad A(s,s') = 0\Rightarrow P(s,s') = 0\ \forall\ s,s' \in \mathcal{S}\ a.s.
\end{equation}
For each $s \in \mathcal{S}$, let $\nu_s = (1-\alpha) e_s^{\tr} ( I - \alpha P )^{-1}$ and $u(s) = \nu_s c$. Then the random vectors $\{ \hat{v}_k , r_k \}_{k=0}^{k_*}$ generated by Algorithm \ref{algBackwardEPE} satisfy the following:
\begin{equation}\label{eqInvariantUnknownQ}
\hat{v}_k(s) + \nu_s r_k = u(s)\ \forall\ k \in \{0,\ldots,k_*\}, s \in \mathcal{S}\ a.s.
\end{equation}
\end{lemma}
\begin{proof}
Fix $s \in \mathcal{S}$. We prove \eqref{eqInvariantUnknownQ} by induction on $k$. For $k = 0$, \eqref{eqInvariantUnknownQ} is immediate, since $\hat{v}_0 = 0_{S \times 1}$ and $r_0 = c$ in Algorithm \ref{algBackwardEPE}. For $k \in [ k_* ]$, the iterative update of Algorithm \ref{algBackwardEPE} implies ($a.s.$)
\begin{align} 
\hat{v}_k(s) + \nu_s r_k  & = \hat{v}_{k-1}(s) + (1-\alpha) r_{k-1}(s_k) 1 ( s = s_k ) \\
& \quad\quad + \sum_{s'=1}^S \nu_s(s') ( r_{k-1}(s') 1 ( s' \neq s_k ) + \alpha \hat{Q}_k(s',s_k) r_{k-1}(s_k) )   \\
&  = \hat{v}_{k-1}(s) + \nu_s r_{k-1} + r_{k-1}(s_k) ( - \nu_s(s_k) + (1-\alpha) 1 ( s = s_k ) + \alpha \nu_s \hat{Q}_k(\cdot,s_k) ) , \quad \label{eqInvariantIndStep}
\end{align}
where for the second equality we added and subtracted $\mu_s(s_k) r_{k-1}(s_k)$ and rearranged the expression. 
Now since $\hat{v}_{k-1}(s) + \nu_s r_{k-1} = u(s)\ a.s.$ by the inductive hypothesis, and since by definition
\begin{align}\label{eqInvariantUnknownQindHypProof}
\nu_s(s_k) - (1-\alpha) 1 ( s = s_k ) & = (1-\alpha) \sum_{t=1}^{\infty} \alpha^t P^t(s,s_k) \\
& = \alpha (1-\alpha)  \sum_{t=0}^{\infty} \alpha^t P^t(s,\cdot) P (\cdot,s_k) = \alpha \nu_s P(\cdot,s_k) ,
\end{align}
it suffices to show $\hat{Q}_k(s',s_k) = P(s',s_k)\ \forall\ s' \in \mathcal{S}\ a.s.$ (since then the term in parentheses in \eqref{eqInvariantIndStep} will be zero). Towards this end, we fix $s' \in \mathcal{S}$ and consider two cases:
\begin{itemize}
\item If $s' \in U_k$, Algorithm \ref{algBackwardEPE} implies $\hat{Q}_k(s',s_k) = \hat{Q}_{k_*}(s',s_k)$ (once we estimate $Q(s',\cdot)$, our estimate remains unchanged). Moreover,  $U_k \subset U_{k_*}$ in Algorithm \ref{algBackwardEPE} (the encountered set only grows), so $s' \in U_{k_*}$, and thus $P(s',s_k) = \hat{Q}_{k_*}(s',s_k)\ a.s.$ by assumption on $P$. Taken together, $\hat{Q}_k(s',s_k) = P(s',s_k)\ a.s.$
\item If $s' \notin U_k$, Algorithm \ref{algBackwardEPE} implies $Q_k(s',s_k) = 0$ (before encountering $s'$, our estimate of $Q(s',\cdot)$ is $0_{1 \times S}$). On the other hand, $N_{in}(s_k) \subset U_k$, so $s' \notin N_{in}(s_k)$ and $A(s',s_k) = 0$ by definition of $N_{in}(s_k)$. Hence, by assumption on $P$, we have $P(s',s_k) = 0\ a.s.$ as well.
\end{itemize}
Thus, $\hat{Q}_k(s',s_k) = P(s',s_k)\ a.s.$ in both cases, completing the proof.
\end{proof}

\begin{remark} \label{remInvariantProofKnownQ}
The \texttt{Approx-Contributions} invariant \eqref{eqInvariantKnownQ} is proven in a similar (but simpler) manner: the base of induction is trivial ($\hat{v}_0(s) + \mu_s r_0 = 0 + \mu_s c = v(s)$); assuming \eqref{eqInvariantKnownQ} holds for $k-1$, one proves it holds for $k$ using the approach of \eqref{eqInvariantIndStep} and \eqref{eqInvariantUnknownQindHypProof} (replacing $\nu_s$ with $\mu_s$ and both $P$ and $\hat{Q}_k$ with $Q$). The crucial idea of \texttt{Backward-EPE} and our analysis is that such an invariant also holds for \texttt{Backward-EPE}, as formalized by preceding lemma. This relies fundamentally on the fact that once $Q(s,\cdot)$ is estimated, the estimate is retained for the duration of the algorithm; otherwise, the logic of the first bullet in the proof above fails. This algorithmic subtlety allows us to prove analogues of existing results for \texttt{Approx-Contributions}; see Remark \ref{remCompCompKnownQ}. 
\end{remark}

\section{Proof of Theorem \ref{thmAccuracy}} \label{proofThmAccuracy}

Fix $s \in \mathcal{S}$ and observe that the $\overline{Q}$-invariant \eqref{eqInvariantOverUnder}, the termination criteria $\|r_k\|_{\infty} \leq \epsilon$ of \texttt{Backward- EPE}, and the fact that $\sum_{s=1}^S \overline{\mu_s}(s) = 1$ by definition together imply 
\begin{equation}
| \hat{v}_{k_*}(s) - \overline{v}(s) | = \overline{\mu_s} r_{k_*} \leq \epsilon .
\end{equation}
Since this inequality holds uniformly in $s$, we can then write
\begin{equation}\label{eqAccProof1first}
\P ( \| \hat{v}_{k_*} - v \|_{\infty} \geq 2\epsilon ) \leq \P ( \| \overline{v} - v \|_{\infty} \geq \epsilon ), 
\end{equation}
so we aim to show the right side is bounded by $\delta$ whenever $n \geq n^*(\epsilon,\delta)$. Towards this end, we begin by deriving a pointwise bound for $\| \overline{v} - v \|_{\infty}$. First, fix $T \in \N$ and observe
\begin{equation}\label{eqAccProof1cvx}
\| \overline{v} - v \|_{\infty} \leq (1-\alpha) \sum_{t=1}^{\infty} \alpha^t \| ( \overline{Q}^t - Q^t ) c \|_{\infty} \leq (1-\alpha) \sum_{t=1}^{T-1} \alpha^t \| ( \overline{Q}^t - Q^t  ) c \|_{\infty} + 2 \| c \|_{\infty} \alpha^T ,
\end{equation}
where the first inequality is convexity and the second holds since by row stochasticity of $\overline{Q}$ and $Q$,
\begin{align}\label{eqVhatMinVtail}
(1-\alpha)  \sum_{t=T}^{\infty} \alpha^t \| ( \overline{Q}^t - Q^t  ) c \|_{\infty} & \leq (1-\alpha)  \sum_{t=T}^{\infty} \alpha^t ( \| \overline{Q}^t c \|_{\infty} + \| Q^t c \|_{\infty} ) \\
& \leq 2 \| c \|_{\infty} (1-\alpha)  \sum_{t=T}^{\infty} \alpha^t =2 \| c \|_{\infty} \alpha^T .
\end{align}
Now for large enough $T$, the bound in \eqref{eqVhatMinVtail} falls below $\epsilon/2$; in particular,
\begin{equation}\label{eqChoiceOfT}
T \geq \frac{ \log ( 4 \| c \|_{\infty} / \epsilon ) }{1-\alpha} \quad \Rightarrow \quad 2 \| c \|_{\infty} \alpha^T \leq 2 \| c \|_{\infty} e^{-(1-\alpha)T} \leq \frac{\epsilon}{2} .
\end{equation}
Furthermore, for the $t$-th summand in \eqref{eqAccProof1cvx}, we can use the triangle inequality to write
\begin{equation}\label{eqSummandTriangle}
\| ( \overline{Q}^t  - Q^t ) c \|_{\infty} \leq \| \overline{Q} ( \overline{Q}^{t-1} - Q^{t-1} ) c \|_{\infty} + \| ( \overline{Q} - Q ) Q^{t-1} c \|_{\infty} .
\end{equation}
For the first summand in \eqref{eqSummandTriangle}, we have by convexity and row stochasticity,
\begin{align}\label{eqSummandTriangleTerm1}
 \| \overline{Q} ( \overline{Q}^{t-1} - Q^{t-1} ) c \|_{\infty} & \leq \max_{s \in \mathcal{S}} \sum_{s'=1}^S \overline{Q}(s,s') | ( \overline{Q}^{t-1}(s',\cdot) - Q^{t-1}(s',\cdot) ) c | \\
& \leq \| ( \overline{Q}^{t-1}  - Q^{t-1} ) c \|_{\infty} .
\end{align}
We can then combine the previous two inequalities and iterate to obtain
\begin{align}
\| ( \overline{Q}^t  - Q^t ) c \|_{\infty} \leq \sum_{\tau=1}^t \| ( \overline{Q} - Q ) Q^{\tau-1} c \|_{\infty} & \leq t \max_{\tau \in [t]} \| ( \overline{Q} - Q ) Q^{\tau-1} c \|_{\infty} \\
& \leq t \max_{\tau \in [T]} \| ( \overline{Q} - Q ) Q^{\tau-1} c \|_{\infty} .
\end{align}
Since this holds uniformly in $t$, we have
\begin{align}\label{eqMaxTandGeoSeries}
(1-\alpha) \sum_{t=1}^{T-1} \alpha^t \| ( \overline{Q}^t  - Q^t ) c \|_{\infty} & \leq \max_{\tau \in [T]} \| ( \overline{Q} - Q ) Q^{\tau-1} c \|_{\infty} (1-\alpha) \sum_{t=1}^{\infty} \alpha^t t \\
& = \max_{\tau \in [T]} \| ( \overline{Q} - Q ) Q^{\tau-1} c \|_{\infty} \frac{\alpha}{1-\alpha} .
\end{align}
To summarize, for $T$ as in \eqref{eqChoiceOfT} we have shown
\begin{equation}\label{eqAccProof1finalPtwise}
\| \overline{v} - v \|_{\infty} \leq \max_{\tau \in [T]} \| ( \overline{Q} - Q ) Q^{\tau-1} c \|_{\infty} \frac{\alpha}{1-\alpha} + \frac{\epsilon}{2} ,
\end{equation}
and so, by the union bound,
\begin{equation}\label{eqPacGuaranteeToTinfNorms}
\P ( \| \overline{v} - v \|_{\infty} \geq \epsilon )  \leq \sum_{t=1}^T \P \left( \| ( \overline{Q} - Q ) Q^{t-1} c \|_{\infty} \geq \frac{\epsilon(1-\alpha)}{2\alpha} \right) .
\end{equation}
Now consider the $t$-th summand in \eqref{eqPacGuaranteeToTinfNorms}. Since $\overline{Q}$ and $\tilde{Q}$ have the same distribution, we can write
\begin{equation}\label{eqAccProof1BarToTilde}
\P \left( \| ( \overline{Q} - Q ) Q^{t-1} c \|_{\infty} \geq \frac{\epsilon(1-\alpha)}{2\alpha}  \right) = \P \left( \| ( \tilde{Q} - Q ) Q^{t-1} c \|_{\infty} \geq \frac{\epsilon(1-\alpha)}{2\alpha} \right).
\end{equation}
To bound the right side of \eqref{eqAccProof1BarToTilde}, we first define $d_{t-1} = Q^{t-1} c$ and observe that for any $s \in \mathcal{S}$,
\begin{equation}
\tilde{Q}(s,\cdot) Q^{t-1} c = \sum_{s' =1}^S \tilde{Q}(s,s') d_{t-1}(s') = \sum_{s' =1 }^S \left( \frac{1}{n} \sum_{i=1}^n 1 ( Y_{s,i} = s' ) \right) d_{t-1}(s') = \frac{1}{n} \sum_{i=1}^n d_{t-1} ( Y_{s,i} ) .
\end{equation}
Moreover, for any $s \in \mathcal{S}, i \in [n]$ we have
\begin{equation}
Q(s,\cdot) Q^{t-1} c = \sum_{s'=1}^s Q(s,s') d_{t-1}(s') = \sum_{s'=1}^s \P ( Y_{s,i} = s' ) d_{t-1}(s')  = \E d_{t-1} ( Y_{s,i} ) .
\end{equation}
Combining the previous two equations, we obtain
\begin{equation}
\| ( \tilde{Q} - Q ) Q^{t-1} c \|_{\infty} = \max_{s \in \mathcal{S}} \left| \frac{1}{n} \sum_{i=1}^n ( d_{t-1} ( Y_{s,i} ) - \E d_{t-1} ( Y_{s,i} ) ) \right| .
\end{equation}
Hence, using the previous equation, and again the union bound, we obtain
\begin{align}\label{eqPacGuaranteeToTinfNorms_s}
& \P \left( \| ( \tilde{Q} - Q ) Q^{t-1} c \|_{\infty} \geq \frac{\epsilon(1-\alpha)}{2\alpha} \right) \\
& \quad\quad \leq \sum_{s=1}^S \P \left( \left| \frac{1}{n} \sum_{i=1}^n ( d_{t-1} ( Y_{s,i} ) - \E d_{t-1} ( Y_{s,i} ) ) \right| \geq \frac{\epsilon(1-\alpha)}{2\alpha} \right) .
\end{align}
Now fix $s \in \mathcal{S}$. Recall $\{ Y_{s,i} \}_{i=1}^n$ are independent, and thus $\{ d_{t-1}(Y_{s,i}) \}_{i=1}^n$ are independent as well. Moreover, $d_{t-1}(Y_{s,i})$ takes values in $[0,\|c\|_{\infty} ]$. 
Thus, we can use the Chernoff bound \eqref{eqChernoffAbs} to obtain
\begin{align}
& \P \left( \left| \frac{1}{n} \sum_{i=1}^n ( d_{t-1} ( Y_{s,i} ) - \E d_{t-1} ( Y_{s,i} ) ) \right| \geq \frac{\epsilon(1-\alpha)}{2\alpha} \right) \\
& \quad\quad = \P \left( \left|  \sum_{i=1}^n \left( \frac{d_{t-1} ( Y_{s,i} )}{\|c\|_{\infty}} - \frac{\E d_{t-1} ( Y_{s,i} )}{\|c\|_{\infty}} \right) \right| \geq \frac{n \epsilon(1-\alpha)}{2 \|c\|_{\infty} \alpha} \right) \\
& \quad\quad \leq  2 \exp \left( - \frac{n \epsilon^2 (1-\alpha)^2 }{ 2 \| c \|_{\infty}^2 \alpha^2 } \right) \leq \frac{\delta}{ST} , \label{eqPacGuaranteeToTinfNorms_fin} 
\end{align}
where the final inequality holds assuming we choose $T$ as small as possible in \eqref{eqChoiceOfT} and by the assumption on $n$ in the statement of the theorem. Combining \eqref{eqAccProof1first}, \eqref{eqPacGuaranteeToTinfNorms}, \eqref{eqAccProof1BarToTilde}, \eqref{eqPacGuaranteeToTinfNorms_s}, and \eqref{eqPacGuaranteeToTinfNorms_fin} implies the theorem.

\begin{remark} \label{remOverOrUnder_acc1}
It may seem wasteful that we use the $\overline{Q}$-invariant instead of the $\underline{Q}$-invariant for Theorem \ref{thmAccuracy}, since $\underline{Q}$ fills unestimated rows of $\hat{Q}_{k_*}$ with the actual rows of $Q$, and thus $\underline{v}$ should be a better estimate of $v$. We explain this choice as follows. First note that by the arguments in the proof, bounding $\| \underline{v} - v \|_{\infty}$ amounts to bounding $\| ( \underline{Q} - Q ) Q^{t-1} c \|_{\infty}$. It is tempting to use the union bound to bound such terms as
\begin{align} \label{eqAcc1whyUnder}
\P \left( \| ( \underline{Q} - Q ) Q^{t-1} c \|_{\infty} \geq \eta \middle| U_{k_*} \right) 
&  \leq \sum_{s \in U_{k_*}} \P \left( \left| \frac{1}{n} \sum_{i=1}^n ( d_{t-1}(X_{s,i}) - \E d_{t-1}(X_{s,i})) \right|\geq \eta\middle| U_{k_*} \right) .
\end{align}
The issue with this approach is that there is a complicated dependence between $\{ X_{s,i} \}_{i=1}^n$ and $U_{k_*}$ in Algorithm \ref{algBackwardEPE}, so we cannot use standard concentration inqualities for the right side of \eqref{eqAcc1whyUnder}. We also note that we replace $\| ( \overline{Q} - Q ) Q^{t-1} c \|_{\infty}$ by $\| ( \tilde{Q} - Q ) Q^{t-1} c \|_{\infty}$ in the proof of Theorem \ref{thmAccuracy} owing to a similar issue.
\end{remark}

\begin{remark} \label{remAccProof1Crand}
This proof assumes the cost vector $c$ is deterministic; in the setting of Theorem \ref{thmEncountered}, the cost vector $C$ is random. In the latter case, we can replace $\P(\cdot)$ by $\P(\cdot|C)$ but otherwise follow the same proof to obtain $\P(\|\hat{v}_{k_*}-v\|_{\infty} \geq 2 \epsilon |C) \leq \delta\ a.s.$ and then average over $C$ to obtain the same result, assuming the lower bound on $n$ (which depends on $\|C\|_{\infty}$) holds almost surely.
\end{remark}

\section{Proof of Theorem \ref{thmEncountered}} \label{proofThmEncountered}

As for Theorem \ref{thmAccuracy}, we exploit the $\overline{Q}$-invariant \eqref{eqInvariantOverUnder} (note we proved Lemma \ref{lemInvariantUnknownQ} for fixed $c$ but the same arguments hold for random $C$ owing to their almost-sure nature). First observe that for any $s \in \mathcal{S}$,
\begin{equation}\label{eqInvariantToSk}
\overline{v}(s) \geq \hat{v}_{k_*}(s) = (1-\alpha) \sum_{k=1}^{k_*} r_{k-1}(s) 1 ( s = s_k ) \geq \epsilon (1-\alpha) \sum_{k=1}^{k_*}  1 ( s = s_k ) ,
\end{equation}
where the first inequality holds by the $\overline{Q}$-invariant \eqref{eqInvariantOverUnder}, the equality by Algorithm \ref{algBackwardEPE}, and the second inequality by definition of $k_*$. On the other hand, we have
\begin{equation}
|U_{k_*}| = | \cup_{s=1}^{k_*} N_{in}(s_k) | \leq \sum_{k=1}^{k_*} d_{in}(s_k) = \sum_{k=1}^{k_*} \sum_{s=1}^S d_{in}(s) 1 ( s = s_k ) = \sum_{s=1}^S d_{in}(s) \sum_{k=1}^{k_*} 1 ( s = s_k ) .
\end{equation}
Combining the previous two inequalities and taking expectation, we have therefore shown
\begin{equation}\label{eqBeforeExpectedVover}
\E |U_{k_*}| \leq \frac{1}{\epsilon(1-\alpha)} \sum_{s=1}^S d_{in}(s) \E \overline{v}(s) .
\end{equation}
Now consider $\E \overline{v}(s)$. By definition \eqref{eqOverDefn},
\begin{equation} \label{eqExpectedVover} 
\E \overline{v}(s) = \E \overline{\mu_s} C = (1-\alpha) \sum_{t=0}^{\infty} \alpha^t \E \overline{Q}^t(s,\cdot) C = (1-\alpha) \sum_{t=0}^{\infty} \alpha^t \E [ \E [ \overline{Q}^t(s,\cdot) | C ] C ] .
\end{equation}
Now after realizing $C$, we fill some rows of $\overline{Q}$ with samples generated during the algorithm and other rows with samples generated offline; in contrast, all rows of $\tilde{Q}$ are filled with offline samples. But in either case, these samples have the same distribution, so we can replace $\overline{Q}$ by $\tilde{Q}$ in the previous equation. Moreover, $\tilde{Q}$ is independent of the random variables in Algorithm \ref{algBackwardEPE}, including $r_0 = C$. In summary,
\begin{equation}\label{eqEncounteredOverToTilde}
\E [ \overline{Q}^t(s,\cdot) | C ] = \E [ \tilde{Q}^t(s,\cdot) | C ] = \E [ \tilde{Q}^t(s,\cdot) ] .
\end{equation}
Combining the previous two equations and using the assumption on $C$, we obtain
\begin{align}
\E \overline{v}(s) & = (1-\alpha) \sum_{t=0}^{\infty} \alpha^t \E [  \tilde{Q}^t(s,\cdot)  ] \E [ C] \leq (1-\alpha) \sum_{t=0}^{\infty} \alpha^t \E [  \tilde{Q}^t(s,\cdot)  ] \bar{c} 1_{S \times 1} = \bar{c} ,
\end{align}
where the final equality holds by row stochasticity of $\tilde{Q}$. Substituting into \eqref{eqBeforeExpectedVover} completes the proof.

\begin{remark} \label{remOverOrUnder_enc}
Note this approach fails if we use the $\underline{Q}$-invariant instead of the $\overline{Q}$-invariant. In particular, we cannot express $\E [ \underline{Q}^t(s,\cdot) | C ]$ as deterministic in \eqref{eqEncounteredOverToTilde}, since $C$ influences which states are encountered during the algorithm and thus influences which rows of $\underline{Q}$ are estimates and which are exact. This illustrates the utility of the $\overline{Q}$-invariant: it allows us to ``decorrelate'' the estimated transition matrix from the cost vector, i.e.\ to obtain $\E [ \overline{Q}^t(s,\cdot)  C ] = \E [ \tilde{Q}^t(s,\cdot) ] \E [ C ]$. In the current work, this is our only use of this decorrelation trick, but it may useful in analyses of algorithms like \texttt{Backward-EPE} (e.g.\ those discussed in Section \ref{secFuture}).
\end{remark}
\begin{remark} \label{remCompCompKnownQ}
The preceding proof is similar to the proof of Theorem 2 in \cite{lofgren2013personalized}, which considers the expected computational complexity of \texttt{Approx-Contributions} when $C \sim \{ e_s \}_{s=1}^S$ uniformly. In fact, \cite{lofgren2013personalized} uses the \texttt{Approx-Contributions} invariant \eqref{eqInvariantKnownQ} but otherwise follows the same logic leading to \eqref{eqBeforeExpectedVover}; since $\mu_s$ is deterministic in the \texttt{Approx-Contributions} setting, one immediately obtains $\E v(s) = \mu_s \E C = \mu_s 1_{S \times 1} / S = 1 / S$ in this case. Similarly, \cite{andersen2008local} provides an instance bound on $k_*$ for fixed $c$ of the form $c = e_{s^*}$; the proof uses \eqref{eqInvariantKnownQ} and the logic of \eqref{eqInvariantToSk} to obtain $v(s) \geq \epsilon(1-\alpha) \sum_{k=1}^{k_*} 1 ( s= s_k )$, then sums over $s$ to obtain $k_* \leq \| v \|_1 / ( \epsilon(1-\alpha) )$. 
\end{remark}

\section{Proof of Corollary \ref{corHnonzeros}} \label{proofCorHnonzeros}

Though we stated Theorem \ref{thmAccuracy} in the case of a deterministic cost vector $c$, it also holds for $C$ if the lower bound on $n$ holds almost surely (see Remark \ref{remAccProof1Crand}). Moreover, by assumption on $C$, $\|C\|_{\infty} = 1$ pointwise and thus $n^*(\epsilon,\delta)$ is deterministic; paired with the assumption on $\alpha,\delta,\epsilon$, we have $n^*(\epsilon,\delta) = O(\log S)$. Thus, the expected sample complexity of \texttt{Backward-EPE} is $\E [  |U_{k_*}| n^*(\epsilon,\delta) ] = O ( \E [|U_{k_*}| ] \log S)$. Again using the assumption on $C$, $\E C(s) = H / S\ \forall\ s \in \mathcal{S}$, so we can apply Theorem \ref{thmEncountered} with $\bar{c} = H/S$ to obtain $\E|U_{k_*}| = O ( H \bar{d} )$. Finally, since $U_{k_*} \subset \mathcal{S}$, we can sharpen this to obtain $\E|U_{k_*}| = O ( \min \{ H \bar{d} , S \} )$.  

\section{Proof of Theorem \ref{thmAccuracy2}} \label{proofThmAccuracy2}

Define $\underline{Q}, \underline{v}$ as in \eqref{eqUnderDefn}. We also define the events
\begin{gather}
E_1 = \cup_{s=1}^S \left\{ | \hat{v}_{BD}(s) - v(s) | \geq \epsilon_{rel} v(s) + \epsilon_{abs} \right\}  , \\
E_{2,s} =  \left\{ | \underline{v}(s) - v(s) | \geq \frac{\epsilon_{rel}}{2} v(s) + \frac{\epsilon_{abs}}{2} \right\} , \quad E_2 = \cup_{s=1}^S E_{2,s} , \\
E_{3,s} =  \left\{ | \hat{v}_{BD}(s) -\underline{v}(s) | \geq \frac{\epsilon_{rel}}{2} v(s) + \frac{\epsilon_{abs}}{2} \right\}, \quad E_3 = \cup_{s=1}^S E_{3,s}  .
\end{gather}
Further, let $\mathcal{G} = \sigma ( \{ \hat{v}_k , r_k , U_k , \hat{Q}_k , s_{k+1} \}_{k=0}^{k_*} )$ denote $\sigma$-algebra generated by the random variables in the Algorithm \ref{algBackwardEPE} subroutine of Algorithm \ref{algBidirectionalEPE}. Note in particular that $\underline{Q}$ is $\mathcal{G}$-measurable, and thus $\underline{v}$ is $\mathcal{G}$-measurable; consequently, $E_{2,s} \in \mathcal{G}$. Using these definitions, we state two key lemmas.
\begin{lemma} \label{lemAccProof2bw}
For $n_B$ as in the theorem statement, $\P ( E_2 ) \leq \delta / 2$.
\end{lemma}
\begin{lemma} \label{lemAccProof2fw}
For $n_F$ as in the theorem statement and any $s \in \mathcal{S}$, $\P ( E_{3,s} | \mathcal{G} ) 1 ( E_{2,s}^C ) \leq \delta / (2S)\ a.s.$
\end{lemma}

Before proving the lemmas, we show that they imply the theorem. Towards this end, first note $E_1 \subset E_2 \cup E_3$ by the triangle inequality, so $E_1 \cap E_2^C \subset E_3 \cap E_2^C$. Consequently,
\begin{equation}
\P ( E_1 ) = \P ( E_1 \cap E_2 ) + \P ( E_1 \cap E_2^C ) \leq \P ( E_2 ) + \P ( E_3 \cap E_2^C ) .
\end{equation}
Furthermore, by the union bound and monotonicity, we have
\begin{equation}
\P ( E_3 \cap E_2^C ) \leq \sum_{s=1}^S \P ( E_{3,s} \cap E_2^C ) \leq \sum_{s=1}^S \P ( E_{3,s} \cap E_{2,s}^C ) .
\end{equation}
Now fix $s \in \mathcal{S}$. Then since $E_{2,s}^C \in \mathcal{G}$, we can write
\begin{equation}
\P ( E_{3,s} \cap E_{2,s}^C ) = \E [ \P ( E_{3,s} | \mathcal{G} ) 1 ( E_{2,s}^C ) ] .
\end{equation}
Combining the previous three inequalities with the two lemmas, we obtain
\begin{equation}
\P ( E_1 ) \leq \P ( E_2 ) + \sum_{s=1}^S \E [ \P ( E_{3,s} | \mathcal{G} ) 1 ( E_{2,s}^C ) ] \leq \delta ,
\end{equation}
and by definition of $E_1$, the theorem follows. We next return to prove the lemmas.

\subsection{Proof of Lemma \ref{lemAccProof2bw}}

First, we define the constants
\begin{equation}
\bar{T} = \ceil*{ \frac{\log ( 2 \| c \|_{\infty} / \epsilon_{abs} ) }{1-\alpha} } , \quad \lambda = \frac{\log ( 1 + \epsilon_{rel} / 2 ) }{ \bar{T} } .
\end{equation}
Next, we prove the following implication:
\begin{equation}\label{eqRelErrTildeValueImp}
| \underline{Q}(s,s') - Q(s,s') | \leq \lambda Q(s,s')\ \forall\ s,s' \in \mathcal{S} \quad \Rightarrow \quad | \underline{v}(s) - v(s) | \leq \frac{\epsilon_{rel}}{2} v(s) + \frac{\epsilon_{abs}}{2}\ \forall\ s \in \mathcal{S} .
\end{equation}
Assume the left side of \eqref{eqRelErrTildeValueImp} holds and fix $s \in \mathcal{S}$. Then clearly
\begin{gather}
 (1-\alpha) \sum_{t=\bar{T}}^{\infty} \alpha^t \underline{Q}^t(s,\cdot) c \leq  (1-\alpha) \sum_{t=\bar{T}}^{\infty} \alpha^t \| c \|_{\infty}  = \alpha^{\bar{T}} \| c \|_{\infty}  \leq e^{-(1-\alpha) \bar{T}} \| c \|_{\infty}  \leq \frac{ \epsilon_{abs}}{2}  \\
\Rightarrow \underline{v}(s) = (1-\alpha) \sum_{t=0}^{\infty} \alpha^t \underline{Q}^t(s,\cdot) c \leq (1-\alpha) \sum_{t=0}^{\bar{T}-1} \alpha^t \underline{Q}^t(s,\cdot) c + \frac{ \epsilon_{abs}}{2} . \label{eqRelErrTildeValueUpperSum}
\end{gather}
We next upper bound the term $\underline{Q}^t(s,\cdot) c$ in the $t$-th summand of \eqref{eqRelErrTildeValueUpperSum}. For $t =0$, this term is simply $c(s)$. For $t=1$, the left side of \eqref{eqRelErrTildeValueImp} implies
\begin{equation}
\underline{Q}(s,\cdot) c = \sum_{s'=1}^S \underline{Q}(s,s') c(s') \leq (1+\lambda) \sum_{s'=1}^S Q(s,s') c(s') = (1+\lambda) Q(s,\cdot) c .
\end{equation}
Finally, for $t \in \{2,\ldots,\bar{T}-1\}$, the left side of \eqref{eqRelErrTildeValueImp} similarly gives
\begin{align}
\underline{Q}^t(s,\cdot) c  & = \sum_{s'\in \mathcal{S}}  \sum_{s_1,\ldots,s_{t-1} \in \mathcal{S}} \underline{Q}(s,s_1) \underline{Q}(s_1,s_2) \cdots \underline{Q}(s_{t-2},s_{t-1}) \underline{Q}(s_{t-1},s')  c(s')   \\
& \leq (1+\lambda)^t \sum_{s'\in \mathcal{S}} \sum_{s_1,\ldots,s_{t-1} \in \mathcal{S}} {Q}(s,s_1) {Q}(s_1,s_2) \cdots {Q}(s_{t-2},s_{t-1}) {Q}(s_{t-1},s') c(s') \\
& = (1+\lambda)^t Q^t(s,\cdot) c .
\end{align}
In summary, we have shown $\underline{Q}^t(s,\cdot) c \leq (1+\lambda)^t Q^t(s,\cdot) c\ \forall\ t \in \{0,\ldots,\bar{T}-1\}$. Also, for such $t$,
\begin{equation}\label{eqExpLamBarT}
(1+\lambda)^t \leq (1+\lambda)^{\bar{T}} \leq e^{ \lambda \bar{T} } \leq 1 + \frac{ \epsilon_{rel} }{2} .
\end{equation}
Combining these observations, we can further bound \eqref{eqRelErrTildeValueUpperSum} as
\begin{equation}\label{eqRelErrTildeValueImpUpperFinal}
\underline{v}(s) \leq \left( 1 + \frac{\epsilon_{rel}}{2} \right) (1-\alpha) \sum_{t=0}^{\bar{T}-1} \alpha^t Q^t(s,\cdot) c + \frac{ \epsilon_{abs}}{2}  \leq \left( 1 + \frac{\epsilon_{rel}}{2} \right) v(s) + \frac{ \epsilon_{abs}}{2} .
\end{equation}
For a lower bound on $\underline{v}(s)$, we similarly have
\begin{align}
\underline{v}(s) & \geq (1-\alpha) \sum_{t=0}^{\bar{T}-1} \alpha^t \underline{Q}^t(s,\cdot) c \geq (1-\lambda)^{\bar{T}}  (1-\alpha) \sum_{t=0}^{\bar{T}-1} \alpha^t Q^t(s,\cdot) c  \\
& = (1-\lambda)^{\bar{T}} \left( {v}(s) - (1-\alpha) \sum_{t=\bar{T}}^{\infty} \alpha^t Q^t(s,\cdot) c \right) \geq  (1-\lambda)^{\bar{T}} \left( {v}(s)  - \frac{ \epsilon_{abs}}{2} \right) .
\end{align}
We now loosen this bound so it matches the form of the upper bound. First, by convexity and \eqref{eqExpLamBarT},
\begin{equation}
2 = 2 \left( \frac{1+\lambda}{2} + \frac{1-\lambda}{2} \right)^{\bar{T}} \leq (1+\lambda)^{\bar{T}} +  (1-\lambda)^{\bar{T}}  \leq \left( 1 +  \frac{ \epsilon_{rel}}{2} \right) + (1-\lambda)^{\bar{T}} ,
\end{equation}
and so $(1-\lambda)^{\bar{T}} \geq 1 - \epsilon_{rel} / 2$. Since also $(1-\lambda)^{\bar{T}}  \leq 1$, we thus obtain
\begin{equation}\label{eqRelErrTildeValueImpLowerFinal}
\underline{v}(s) \geq \left( 1 -  \frac{ \epsilon_{rel}}{2} \right) \left( {v}(s)  - \frac{ \epsilon_{abs}}{2} \right) \geq \left( 1 -  \frac{ \epsilon_{rel}}{2} \right)  {v}(s)  - \frac{ \epsilon_{abs}}{2} . 
\end{equation}
In summary, we have shown that if the left side of \eqref{eqRelErrTildeValueImp} holds, then \eqref{eqRelErrTildeValueImpUpperFinal} and \eqref{eqRelErrTildeValueImpLowerFinal} hold as well. Since \eqref{eqRelErrTildeValueImpUpperFinal} and \eqref{eqRelErrTildeValueImpLowerFinal} together imply the right side of \eqref{eqRelErrTildeValueImp}, \eqref{eqRelErrTildeValueImp} is proven. We can now use \eqref{eqRelErrTildeValueImp} to prove the lemma. First note that \eqref{eqRelErrTildeValueImp} and the union bound together
\begin{align}\label{eqRelErrTildeValueFinalUnion}
\P ( E_2 ) & \leq \P \left( \cup_{s,s' \in \mathcal{S}} \{ | \underline{Q}(s,s') - Q(s,s') | > \lambda Q(s,s') \} \right) \\
& \leq \sum_{s,s' \in \mathcal{S}} \P ( | \underline{Q}(s,s') - Q(s,s') | > \lambda Q(s,s') ) . 
\end{align}
Now for the $(s,s')$-th summand in \eqref{eqRelErrTildeValueFinalUnion}, we first note
\begin{align}\label{eqRelErrPassToTilde}
\P ( | \underline{Q}(s,s') - Q(s,s') | > \lambda Q(s,s')  ) & \leq \P ( | \overline{Q}(s,s') - Q(s,s') | > \lambda Q(s,s')  ) \\
& = \P ( | \tilde{Q}(s,s') - Q(s,s') | > \lambda Q(s,s')  )
\end{align}
where the inequality holds since $| \underline{Q}(s,s') - Q(s,s') | \leq | \overline{Q}(s,s') - Q(s,s') |$ pointwise by \eqref{eqOverDefn}-\eqref{eqUnderDefn} and uses convexity, and the equality holds since $\overline{Q}$ and $\tilde{Q}$ have the same distribution. Substituting into \eqref{eqRelErrTildeValueFinalUnion}, we obtain
\begin{equation}\label{eqRelErrTildeValueFinalUnion_2}
\P(E_2) \leq \sum_{s,s' \in \mathcal{S}} \P ( | \tilde{Q}(s,s') - Q(s,s') | > \lambda Q(s,s') ) ,
\end{equation}
so our goal is to bound each summand in \eqref{eqRelErrTildeValueFinalUnion_2} by $\delta / (2S^2)$. If $Q(s,s') = 0$, this is trivial; if instead $Q(s,s') > 0$, the Chernoff bound \eqref{eqChernoffRel} implies
\begin{align}
\P ( | \tilde{Q}(s,s') - Q(s,s') | > \lambda Q(s,s') ) 
& \leq 2 \exp \left( - \frac{n_B \lambda^2 \min_{i,j \in \mathcal{S}: Q(i,j) > 0} Q(i,j) }{3} \right) \leq \frac{\delta}{2 S^2} ,
\end{align}
where the final inequality holds by assumption on $n_B$.

\subsection{Proof of Lemma \ref{lemAccProof2fw}}

Fix $s \in \mathcal{S}$. Then by definition of $E_{2,s}, E_{3,s}$, we aim to show
\begin{equation}\label{eqRelErrFwImp}
| \underline{v}(s) - v(s) | < \frac{\epsilon_{rel}}{2} v(s) + \frac{\epsilon_{abs}}{2} \quad \Rightarrow \quad \P \left( | \hat{v}_{BD}(s) - \underline{v}(s) | \geq \frac{\epsilon_{rel}}{2} v(s) + \frac{\epsilon_{abs}}{2} \middle| \mathcal{G} \right) \leq \frac{\delta}{2S}\ a.s.
\end{equation}
Assume the left side of \eqref{eqRelErrFwImp} holds. Recall that by Algorithm \ref{algBidirectionalEPE} and the $\underline{Q}$-invariant \eqref{eqInvariantOverUnder},
\begin{equation}\label{eqRelErrDefnInvariant}
\hat{v}_{BD}(s) = \hat{v}_{k_*}(s) + \frac{1}{n_F} \sum_{i=1}^{n_F} r_{k_*} ( Z_{s,i} ) , \quad \underline{v}(s) = \hat{v}_{k_*}(s) +  \underline{\mu_s} r_{k_*} = v_{k_*}(s) + \frac{1}{n_F} \sum_{i=1}^{n_F} \E [ r_{k_*} ( Z_{s,i} ) | \mathcal{G} ] .
\end{equation}
Consequently, defining $\bar{Z}_s = \sum_{i=1}^{n_F} r_{k_*} ( Z_{s,i} ) / \epsilon$, we have
\begin{align}\label{eqPerfectSamplingRewrite}
\P \left( | \hat{v}_{BD}(s) - \underline{v}(s) | > \frac{\epsilon_{rel}}{2} v(s) + \frac{\epsilon_{abs}}{2} \middle| \mathcal{G} \right) & = \P \left( | \bar{Z}_s - \E [ \bar{Z}_s | \mathcal{G} ] | > \frac{n_F}{\epsilon} \left( \frac{\epsilon_{rel}}{2} v(s) + \frac{\epsilon_{abs}}{2}  \right) \middle| \mathcal{G} \right) .
\end{align}
Note that conditioned on $\mathcal{G}$, $\bar{Z}_s$ is a sum of independent $[0,1]$-valued random variables, so the Chernoff bounds from Appendix \ref{appChernoff} apply. We apply a different bound for each of the following two cases:
\begin{itemize}
\item $\E [ \bar{Z}_s | \mathcal{G} ] < n_F \epsilon_{abs} / ( 12 \epsilon )$: Here we bound the right side of \eqref{eqPerfectSamplingRewrite} as
\begin{align}
& \P \left( | \bar{Z}_s - \E [ \bar{Z}_s | \mathcal{G} ] | > \frac{n_F}{\epsilon} \left( \frac{\epsilon_{rel}}{2} v(s) + \frac{\epsilon_{abs}}{2}  \right) \middle| \mathcal{G} \right) \leq \P \left( | \bar{Z}_s - \E [ \bar{Z}_s | \mathcal{G} ] | >\frac{n_F \epsilon_{abs}}{2 \epsilon}  \middle| \mathcal{G} \right) \\
& \quad\quad = \P \left( \bar{Z}_s - \E [ \bar{Z}_s | \mathcal{G} ] >\frac{n_F \epsilon_{abs}}{2 \epsilon}  \middle| \mathcal{G} \right)  + \P \left( \E [ \bar{Z}_s | \mathcal{G} ] - \bar{Z}_s >\frac{n_F \epsilon_{abs}}{2 \epsilon}  \middle| \mathcal{G} \right) \\
& \quad\quad \leq \P \left( \bar{Z}_s > \frac{n_F \epsilon_{abs}}{2 \epsilon}  \middle| \mathcal{G} \right) ,
\end{align}
where the first inequality and the equality are immediate, and the second inequality holds since, by assumption on $\E [ \bar{Z}_s | \mathcal{G} ]$, $\E [ \bar{Z}_s | \mathcal{G} ] - \bar{Z}_s \leq \E [ \bar{Z}_s | \mathcal{G} ] < n_F \epsilon_{abs} / ( 12 \epsilon ) < n_F \epsilon_{abs} / ( 2 \epsilon )$, so $\E [ \bar{Z}_s | \mathcal{G} ] - \bar{Z}_s > n_F \epsilon_{abs} / ( 2 \epsilon )$ cannot occur. For the remaining term, recall $\E [ \bar{Z}_s | \mathcal{G} ] < (1/6) \times n_F \epsilon_{abs} / ( 2 \epsilon )$, so we can use the Chernoff bound \eqref{eqChernoffLargeEta}. Combined with the above, we obtain
\begin{align}
\P \left( | \bar{Z}_s - \E [ \bar{Z}_s | \mathcal{G} ] | > \frac{n_F}{\epsilon} \left( \frac{\epsilon_{rel}}{2} v(s) + \frac{\epsilon_{abs}}{2}  \right) \middle| \mathcal{G} \right)  & \leq \P \left( \bar{Z}_s > \frac{n_F \epsilon_{abs}}{2 \epsilon}  \middle| \mathcal{G} \right) \\
&  \leq 2^{-n_F \epsilon_{abs} / ( 2 \epsilon ) } \leq \frac{\delta}{4S} ,
\end{align}
where the final inequality holds since, by the theorem statement,
\begin{equation}\label{eqRelErrFwJustOneNfBound}
n_F \geq \frac{ 324 \epsilon \log ( 4 S / \delta ) }{ \epsilon_{rel}^2 \epsilon_{abs} } = \frac{162}{ \epsilon_{rel}^2 \log_2 e } \frac{ 2 \epsilon \log_2(4S/\delta) }{ \epsilon_{abs} } \geq \frac{ 2 \epsilon \log_2(4S/\delta) }{ \epsilon_{abs} } .
\end{equation}

\item $\E [ \bar{Z}_s | \mathcal{G} ] \geq n_F \epsilon_{abs} / ( 12 \epsilon )$: We first observe
\begin{equation}
\underline{v}(s) < \left( 1 + \frac{ \epsilon_{rel} }{2} \right) v(s) + \frac{ \epsilon_{abs} }{2}  \quad \Leftrightarrow \quad \frac{ \underline{v}(s) - \epsilon_{abs} / 2 }{ 1 + \epsilon_{rel} / 2 } < v(s) .
\end{equation}
Consequently, the left side of \eqref{eqRelErrFwImp} implies 
\begin{align}
\frac{\epsilon_{rel}}{2} v(s) + \frac{\epsilon_{abs}}{2} & > \frac{\epsilon_{rel}}{2} \frac{ \underline{v}(s) - \epsilon_{abs} / 2 }{ 1 + \epsilon_{rel} / 2 } + \frac{\epsilon_{abs}}{2} \\
& = \frac{ \epsilon_{rel} \underline{v}(s) }{ 2 + \epsilon_{rel} } + \frac{\epsilon_{abs}}{2} \left( 1 - \frac{\epsilon_{rel}/2}{1+\epsilon_{rel}/2} \right) > \frac{ \epsilon_{rel} \underline{v}(s) }{3} ,
\end{align}
where the final inequality holds by $\epsilon_{rel} \in (0,1)$. Since also $\underline{v}(s) \geq \E [ r_{k_*} ( Z_{s,i} ) | \mathcal{G} ]$ by \eqref{eqRelErrDefnInvariant}, we thus obtain
\begin{equation}
\frac{n_F}{\epsilon} \left( \frac{\epsilon_{rel}}{2} v(s) + \frac{\epsilon_{abs}}{2}  \right) > \frac{n_F}{\epsilon}  \frac{ \epsilon_{rel} \E [ r_{k_*} ( Y_{s,i} ) | \mathcal{G} ] }{3} = \frac{\epsilon_{rel}}{3} \frac{n_F \E [ r_{k_*} ( Y_{s,i} ) | \mathcal{G} ]}{\epsilon} = \frac{\epsilon_{rel}}{3}  \E [ \bar{Z}_s | \mathcal{G} ] .
\end{equation}
Therefore, we can bound the right side of \eqref{eqPerfectSamplingRewrite} as
\begin{align}
& \P \left( | \bar{Z}_s - \E [ \bar{Z}_s | \mathcal{G} ] | > \frac{n_F}{\epsilon} \left( \frac{\epsilon_{rel}}{2} v(s) + \frac{\epsilon_{abs}}{2}  \right) \middle| \mathcal{G} \right) \leq \P \left( | \bar{Z}_s - \E [ \bar{Z}_s | \mathcal{G} ] | > \frac{\epsilon_{rel}}{3}  \E [ \bar{Z}_s | \mathcal{G} ]  \middle| \mathcal{G} \right) \\
& \quad\quad \leq 2 \exp \left( - \frac{ (\epsilon_{rel} / 3)^2 }{3} \E [ \bar{Z}_s | \mathcal{G} ] \right) \leq 2 \exp \left( - \frac{ \epsilon_{rel}^2 }{ 27 } \frac{n_F \epsilon_{abs}}{ 12 \epsilon } \right) \leq \frac{\delta}{2S} ,
\end{align}
where we used the Chernoff bound \eqref{eqChernoffRel}, the $\E [ \bar{Z}_s | \mathcal{G}]$ assumption, and the assumption on $n_F$.
\end{itemize}

\begin{remark} \label{remOverOrUnder_acc2}
While the choice of invariant used to prove Theorems \ref{thmAccuracy} and \ref{thmEncountered} was subtle (see Remarks \ref{remOverOrUnder_acc1} and \ref{remOverOrUnder_enc}), choosing the $\underline{Q}$-invariant for Theorem \ref{thmAccuracy2} is rather obvious, since we explicitly use $\underline{Q}$ in Algorithm \ref{algBidirectionalEPE}.
\end{remark}

\begin{remark} \label{remAccProof2Crand}
The proof of Lemma \ref{lemAccProof2bw} extends to random cost vectors $C$ by replacing $\P(\cdot)$ by $\P(\cdot|C)$ and then averaging over $C$, similar to the proof of Theorem \ref{thmAccuracy} (see Remark \ref{remAccProof1Crand}). Furthermore, recall $r_0 = C$ and thus $C$ is $\mathcal{G}$-measurable by definition of $\mathcal{G}$, so the proof of Lemma \ref{lemAccProof2bw} is identical in the case of random cost $C$. Thus, when $C$ is random, Lemmas \ref{lemAccProof2bw} and \ref{lemAccProof2fw} hold and can be used to prove the theorem as above.
\end{remark}

\section{Alternative approach} \label{appResampling}

The alternative approach is defined in Algorithm \ref{algBackwardEPEresample}. In contrast to \texttt{Backward-EPE}, we estimate $Q(s,s_k)$ as follows at each iteration $k$: for $s \in N_{in}(s_k)$ we draw independent samples $\{ X_{s,i}^k \}_{i=1}^n$ from $Q(s,\cdot)$, and for $s \notin N_{in}(s_k)$ we set $\hat{Q}_k(s,s_k) = 0$; note the estimate of $Q(s,s_k)$ is exact in the latter case owing to \eqref{eqAbsContCond}. We then compute $\hat{v}_k, r_k$ using the update rule from \texttt{Backward-EPE}. Finally, as in \texttt{Backward-EPE}, we terminate when $\|r_k\|_{\infty} \leq \epsilon$.

\begin{algorithm}
\caption{ \texttt{Backward-EPE-Alternative} } \label{algBackwardEPEresample}
\KwIn{Sampler for transition matrix $Q$; cost vector $c$; discount factor $\alpha$; supergraph in-neighbors $\{N_{in}(s)\}_{s=1}^S$; termination parameter $\epsilon$; per-state sample count $n$}
$k = 0$, $\hat{v}_k = 0_{S \times 1}$, $r_k = c$

\While{$\|r_k\|_{\infty} > \epsilon$}{%
	$k \leftarrow k+1$, $s_k \sim \argmax_{s \in \mathcal{S}} r_{k-1}(s)$ uniformly

	\For{$s \in \mathcal{S}$}{
		\lIf{$s \in N_{in}(s_k)$}{$\{ X_{s,i}^k \}_{i=1}^n \sim Q(s,\cdot)$, $\hat{Q}_k(s,s_k) = \frac{1}{n} \sum_{i=1}^n 1 ( X_{s,i}^k= s_k)$ } 
		\lElse{$\hat{Q}_k(s,s_k) = 0$}
		\lIf{$s = s_k$}{$\hat{v}_k(s) = \hat{v}_{k-1}(s) + (1-\alpha) r_{k-1}(s)$, $r_k(s) = \alpha \hat{Q}_k(s,s_k) r_{k-1}(s_k)$}
		\lElse{$\hat{v}_k(s) = \hat{v}_{k-1}(s)$, $r_k(s) = r_{k-1}(s) + \alpha \hat{Q}_k(s,s_k) r_{k-1}(s_k)$}
	}
} 
\KwOut{Estimate $\hat{v}_{k}$ of $v = (1-\alpha) \sum_{t=0}^{\infty} \alpha^t Q^t c$}
\end{algorithm}

We next derive the martingale property mentioned in Section \ref{secFuture}. Toward this end, first let $\mu_s = (1-\alpha) e_s^{\tr} ( I - \alpha Q )^{-1}$ as in Appendix \ref{appComparison} and define $e_k(s) = \hat{v}_k(s) + \mu_s r_k - v(s)$. Note that if $Q$ is known and $\hat{v}_k(s), r_k$ are generated by the existing algorithm \texttt{Approx-Contributions}, then $\hat{v}_k(s) + \mu_s r_k = v(s)$ (see \eqref{eqInvariantKnownQ} in Appendix \ref{appComparison}); thus, $e_k(s)$ is the error process that arises when $Q$ is unknown in Algorithm \ref{algBackwardEPEresample}. Next, define a filtration $\{ \mathcal{F}_k \}_{k=0}^{k_*}$ by $\mathcal{F}_k = \sigma ( \{ \hat{v}_{k'} , r_{k'} , s_{{k'}+1} \}_{k'=0}^k )$, where by $\sigma(\cdot)$ we mean the generated $\sigma$-algebra. Now fix $k \in [k_*], s \in \mathcal{S}$. Then by the iterative update in Algorithm \ref{algBackwardEPEresample}, we have
\begin{align}
e_k(s) & = ( \hat{v}_{k-1}(s) + (1-\alpha) r_{k-1}(s) 1 ( s = s_k ) ) \\
& \quad\quad + \sum_{s'=1}^s \mu_s(s') ( r_{k-1}(s') 1 ( s' \neq s_k) + \alpha \hat{Q}_k(s',s_k) r_{k-1}(s_k) ) - v(s) \\
& = e_{k-1}(s) + r_{k-1}(s_k) \left( - \mu_s(s_k) +  (1-\alpha) 1 ( s = s_k ) + \alpha \sum_{s'=1}^s \mu_s(s') \hat{Q}_k(s',s_k) \right) \label{eqResampleErrorRecursion}
\end{align}
Note that all terms in \eqref{eqResampleErrorRecursion} except $\hat{Q}_k(s',s_k)$ are $\mathcal{F}_{k-1}$-measurable, and therefore
\begin{align}
& \E [ e_k(s) | \mathcal{F}_{k-1} ] - e_{k-1}(s) \\
& \quad\quad =  r_{k-1}(s_k) \left( - \mu_s(s_k) +  (1-\alpha) 1 ( s = s_k ) + \alpha \sum_{s'=1}^s \mu_s(s') \E [ \hat{Q}_k(s',s_k) | \mathcal{F}_{k-1} ]  \right) \\
& \quad\quad = r_{k-1}(s_k) \left( - \mu_s(s_k) +  (1-\alpha) 1 ( s = s_k ) + \alpha \sum_{s'=1}^s \mu_s(s') Q(s',s_k)  \right) = 0,
\end{align}
where the first two equalities hold by Algorithm \ref{algBackwardEPEresample} and the third holds similar to \eqref{eqInvariantUnknownQindHypProof}. Hence, $\E [ e_k(s) | \mathcal{F}_{k-1} ] = e_{k-1}(s)$, i.e.\ $\{ e_k(s) \}_{k=0}^{k_*}$ is a martingale. Also note $e_0(s) = \hat{v}_0(s) + \mu_s r_0 - v(s) = 0 + \mu_s c - v(s) = 0$. Taken together, we conclude $\E e_k(s) = 0$. Thus, by definition of the error process, the \texttt{Approx-Contributions} invariant holds in expectation.

\section{Experimental details} \label{appExperiment}

\noindent \textbf{Generating random problem instances:} To generate $Q$, we elementwise multiply a matrix of independent $\textrm{Uniform}([0,1])$ random variables with a matrix of independent $\textrm{Bernoulli}(p/S)$ random variables, then normalize so that each row sums to $1$. Varying $p$ allows us to control $\bar{d}$; observe in particular that $\E \bar{d} = p$. To generate $c$, we let $c_1$ be a vector of independent $\textrm{Bernoulli}(p/S)$ random variables, $c_2$ a vector of independent $\textrm{Uniform}[0,p/S]$ random variables, and $c = c_1 + c_2$. Note that $\E \|c\|_1 = S ( \frac{p}{S} + \frac{p}{2S} ) = \frac{3p}{2}$ and $\|c\|_{\infty} \in [1,2]$ assuming $c_1 \neq 0$ and $p \leq S$; thus, $\E \|c\|_1 / \| c \|_{\infty} = \Theta(p)$ in this case. Taken together, we (roughly) have $\bar{d} \|c\|_1 / \|c\|_{\infty} = \Theta ( p^2 )$. Note our generation of $Q$ is ill-defined if the Bernoulli matrix has any rows summing to $0$; thus, we resample this matrix until all row sums are positive. We also resample $c_1$ until $c_1 \neq 0_{S \times 1}$ to ensure at least one high-cost state. In practice, $\bar{d} \|c\|_1 / \|c\|_{\infty} \approx \Theta ( p^2 )$ still holds after this resampling.

\noindent \textbf{Figure \ref{figBackwardNumerical} experiment parameters:} We simulate the algorithms for $S \in \{100,200,400,800,1600\}$ and for a variety of $p$. In particular, Case 1 sets $p = 10$ for each $S$, Case 2 sets $p = p(S) = ( 100 S )^{1/4}$, and Case 3 sets $p = p(S) = \sqrt{S}$. Note all three cases yield $p = 10$ when $S = 100$, which is why the $S = 100$ datapoints are similar across cases. For each $S$ and each case of $p$, we run $100$ trials (i.e.\ we generate $100$ different problem instances and run both algorithms for each problem instance). Finally, we set $\alpha = 0.1$ throughout the experiments.

\noindent \textbf{Figure \ref{figBackwardNumerical} algorithmic parameters:} For \texttt{Backward-EPE}, we set $\epsilon = 0.15$ and $n = 20$; for the forward approach, we sample $4$ trajectories of length $\frac{1}{1-\alpha} = 10$ for each state. Thus, \texttt{Backward-EPE} requires $20 S$ samples in the worst case, while the forward approach requires $40 S$ samples in any case. This is why all datapoints in the middle plot of Figure \ref{figBackwardNumerical} lie at or below $\frac{ 20S }{40 S} = 0.5$. We note these algorithmic parameters are not those required analytically (which are too loose in practice), but we find in practice that they yield similar $l_{\infty}$ error.

\noindent \textbf{Figure \ref{figBidirectionalNumerical} experiment parameters:} We simulate the algorithms for $S \in \{100,200,400,800,1600,\\ 3200\}$, generating $Q$ and $c$ as above with $p=10$ (i.e.\ Case 1 from Figure \ref{figBackwardNumerical}). As in Figure \ref{figBackwardNumerical}, we set $\alpha = 0.9$ and conduct $100$ trials.

\noindent  \textbf{Figure \ref{figBidirectionalNumerical} algorithmic parameters:} For the forward approach, we sample $0.05 S$ trajectories of length $\frac{1.5}{1-\alpha} = 15$ for each state; note the number of trajectories and their lengths are both greater than in Figure \ref{figBackwardNumerical}, which we find is necessary to maintain constant relative error. For \texttt{Backward- EPE}, we set $\epsilon = 10/S$ and $n = S$; again, these parameters are modified from Figure \ref{figBackwardNumerical} to maintain constant relative error. For \texttt{Bidirectional-EPE}, we set $n_B = n = S$ and $n_F = 1.5 \sqrt{S}$. Instead of fixing $\epsilon$ (the termination criteria for the \texttt{Backward-EPE} subroutine) \textit{a priori}, we choose it dynamically; in particular, we terminate the subroutine at the first iteration $k$ for which $|U_k| n_B \geq S n_F$. Note this trades off backward and forward sample complexity, i.e.\ we terminate the backward stage when its sample complexity exceeds the complexity of the forthcoming forward stage.

\section{Analysis of forward approach} \label{appStandardApproach}

We recall from Section \ref{secIntro} that the forward approach proceeds as follows. Fix $T \in \N$ and, for each $s \in \mathcal{S}$, sample $m$ length-$T$ trajectories $\{ \{ W_t^{s,i} \}_{t=0}^{T-1} \}_{i=1}^m$ beginning at $s$, and estimate $v(s)$ as
\begin{equation}
\hat{v}_E(s) = \frac{1}{m} \sum_{i=1}^m (1-\alpha) \sum_{t=0}^{T-1} \alpha^t c ( W_t^{s,i} ) .
\end{equation}
(We use the subscript $E$ to distinguish the estimate of this forward approach from the estimates of our algorithms.) To analyze this scheme, we follow the analysis of Proposition 5.4 in \cite{haskell2016empirical}. By the argument leading to \eqref{eqChoiceOfT} in Appendix \ref{proofThmAccuracy} (but with a different constant), we have
\begin{equation}\label{eqChoiceOfTstd}
T \geq \frac{ \log ( 2 \|c\|_{\infty} / \epsilon ) }{1-\alpha}  \quad \Rightarrow \quad | \hat{v}_E(s) - v(s) | \leq | \hat{v}_E(s) - \E \hat{v}_E(s) | + \epsilon ,
\end{equation}
so consequently, for $T$ as in \eqref{eqChoiceOfTstd},
\begin{equation}
\P ( | \hat{v}_E(s) - v(s) | \geq 2 \epsilon ) \leq \P \left( | \hat{v}_E(s) - \E \hat{v}_E(s) | \geq \epsilon \right) .
\end{equation}
Towards further bounding the right side, we write (as in \eqref{eqMaxTandGeoSeries})
\begin{align}\label{eqHaskellMistake_start}
| \hat{v}_E(s) - v(s) | 
& \leq \max_{t \in [T-1] } \left| \frac{1}{m} \sum_{i=1}^m ( c ( W_t^{s,i} ) - \E c ( W_t^{s,i} ) ) \right| \frac{\alpha}{1-\alpha} .
\end{align}
Combining the previous two inequalities, and using the union bound,
\begin{equation} \label{eqHaskellMistake_end}
\P ( | \hat{v}_E(s) - v(s) | \geq 2 \epsilon )   \leq \sum_{t=1}^{T-1} \P \left( \left| \frac{1}{m} \sum_{i=1}^m ( c ( W_t^{s,i} ) - \E c ( W_t^{s,i} ) ) \right| \geq \frac{\epsilon(1-\alpha)}{\alpha} \right) .
\end{equation}
We then apply the Chernoff bound \eqref{eqChernoffAbs} to bound the $t$-th summand by
\begin{align}
\P \left( \left| \sum_{i=1}^m \left( \frac{c ( W_t^{s,i} )}{\|c\|_{\infty}} - \frac{\E c ( W_t^{s,i} ) )}{\|c\|_{\infty}} \right) \right| \geq \frac{m \epsilon(1-\alpha)}{\|c\|_{\infty} \alpha} \right) \leq 2 \exp \left( - \frac{2 m \epsilon^2 (1-\alpha)^2}{ \|c\|_{\infty}^2 \alpha^2 } \right) .
\end{align}
Note this holds uniformly in $t$; also, we can take a union bound over $s \in \mathcal{S}$ to obtain
\begin{equation}
\P ( \| \hat{v}_E - v \|_{\infty} \geq 2 \epsilon ) \leq 2 S T \left( - \frac{2 m \epsilon^2 (1-\alpha)^2}{ \|c\|_{\infty}^2 \alpha^2 } \right) \leq \delta ,
\end{equation}
where the final inequality holds assuming we choose
\begin{equation}
m \geq \frac{ \|c\|_{\infty}^2 \alpha^2  }{ 2 \epsilon^2 (1-\alpha)^2 } \log \left( \frac{2 S T}{\delta} \right) . 
\end{equation}
Note here that $m$ is the number of length-$T$ trajectories sampled from each state. Thus, the overall sample complexity is at least $S T m$, which we can lower bound as
\begin{equation}
S T m \geq \frac{S \|c\|_{\infty}^2 \alpha^2 \log ( 2 \|c\|_{\infty} / \epsilon ) }{ 2 \epsilon^2 (1-\alpha)^3 } \log \left( \frac{2 S}{\delta}  \frac{ \log ( 2 \|c\|_{\infty} / \epsilon ) }{1-\alpha} \right) .
\end{equation}

\section{Chernoff bounds} \label{appChernoff}

The following is a standard concentration of measure result used throughout our analysis.
\begin{theorem}
Let $\{ R_i \}_{i=1}^m$ be independent $[0,1]$-valued random variables and $R = \sum_{i=1}^m R_i$. Then
\begin{align}
\P ( | R - \E R | > \eta ) & \leq 2 \exp ( - 2 \eta^2 / m )\ \forall\ \eta > 0,  \label{eqChernoffAbs} \\
\P ( | R - \E R | > \eta \E R ) & \leq 2 \exp ( - \eta^2 \E R / 3 )\ \forall\ \eta \in (0,1), \label{eqChernoffRel} \\
\P ( R > \eta ) & \leq 2^{-\eta}\ \forall\ \eta > 6 \E R. \label{eqChernoffLargeEta} 
\end{align}
\end{theorem}
\begin{proof}
See e.g.\ Theorem 1.1 in \cite{dubhashi2009concentration}.
\end{proof}

\end{document}